\documentclass[aos,MSNbibl,seceqn,noautosecdot,dvips]{arximspdf}
\usepackage{algorithmic,algorithm}
\usepackage{stfloats}
\usepackage{graphicx}
\usepackage{url,breakurl}

\doi{10.1214/13-AOS1199} 
\volume{42}
\issue{2}
\pubyear{2014}
\firstpage{669}
\lastpage{699}

\makeatletter
  \fnbelowfloat
\newcommand{\rrVert}{\Vert}
\newcommand{\llVert}{\Vert}
\newtheorem{theorem}{Theorem}[section]
\newproclaim{definition}[theorem]{Definition}

\makeatother

\begin{document}
\begin{frontmatter}

\title{Robust subspace clustering}
\runtitle{Robust subspace clustering}

\begin{aug}
\author[A]{\fnms{Mahdi} \snm{Soltanolkotabi}\ead[label=e1]{mahdisol@stanford.edu}\thanksref{t1}},
\author[B]{\fnms{Ehsan} \snm{Elhamifar}\ead[label=e2]{ehsan@eecs.berkeley.edu}}\\
\and
\author[C]{\fnms{Emmanuel J.} \snm{Cand\`{e}s}\corref{}\ead[label=e3]{candes@stanford.edu}\ead[label=u1,url]{http://www.foo.com}\thanksref{t3}}
\runauthor{M. Soltanolkotabi, E. Elhamifar and E. J. Cand\`{e}s}
\affiliation{Stanford University, University of California, Berkeley
and\break  Stanford University}
\address[A]{M. Soltanolkotabi\\
Department of Electrical Engineering\\
Stanford University\\
350 Serra Mall\\
Stanford, California 94305\\
USA\\
\printead{e1}} 
\address[B]{E. Elhamifar\\
EECS Department\\
Trust Center Room 337\\
Cory Hall Engineering Department\\
University of California\\
Berkeley, California 94720-1774\\
USA\\
\printead{e2}}
\address[C]{E. J. Cand\`{e}s\\
Department of Statistics\\
Stanford University\\
390 Serra Mall\\
Stanford, California 94305\\
USA\\
\printead{e3}}
\end{aug}
\thankstext{t1}{Supported by a Benchmark Stanford Graduate Fellowship.}
\thankstext{t3}{Supported in part by AFOSR under Grant FA9550-09-1-0643 and by ONR under Grant N00014-09-1-0258 and by a gift from the Broadcom Foundation.}

\received{\smonth{2} \syear{2013}}
\revised{\smonth{12} \syear{2013}}

%
\begin{abstract}
Subspace clustering refers to the task of finding a multi-subspace
representation that best fits a collection of points taken from a
high-dimensional space. This paper introduces an algorithm inspired
by sparse subspace clustering (SSC) [In \textit{IEEE Conference on Computer Vision and Pattern
Recognition}, \textit{CVPR} (2009) 2790--2797] to cluster noisy
data, and develops some novel theory demonstrating its
correctness. In particular, the theory uses ideas from geometric
functional analysis to show that the algorithm can accurately
recover the underlying subspaces under minimal requirements on their
orientation, and on the number of samples per subspace. Synthetic as
well as real data experiments complement our theoretical study,
illustrating our approach and demonstrating its effectiveness.
\end{abstract}

%

\end{frontmatter}

\setcounter{footnote}{2}
\section{\texorpdfstring{Introduction.}{Introduction}}
\label{secintro}
\subsection{\texorpdfstring{Motivation.}{Motivation}}

In many problems across science and engineering, a fundamental step is
to find a lower-dimensional subspace which best fits a collection of
points taken from a high-dimensional space; this is classically
achieved via Principal Component Analysis (PCA). Such a procedure
makes perfect sense as long as the data points are distributed around
a lower-dimensional subspace, or expressed differently, as long as the
data matrix with points as column vectors has approximately low
rank. A more general model might sometimes be useful when the data
come from a mixture model in which points do not lie around a single
lower-dimensional subspace but rather around a union of
low-dimensional subspaces. For instance, consider an experiment in
which gene expression data are gathered on many cancer cell lines with
unknown subsets belonging to different tumor types. One can imagine
that the expressions from each cancer type may span a distinct lower-dimensional subspace. If the cancer labels were known in advance, one
would apply PCA separately to each group but we here consider the case
where the observations are unlabeled. Thus, the goal in such an example
would be to separate gene expression patterns into different cancer
types if possible. Finding the components of the mixture and
assigning each point to a fitted subspace is called subspace
clustering. Even when the mixture model holds, the full data matrix
may not have low rank at all, a situation which is very different from
that where PCA is applicable.


In recent years, numerous algorithms have been developed for subspace
clustering and applied to various problems in computer vision/machine
learning \cite{vidaltutorial} and data mining
\cite{parsons2004subspace}. At the time of this writing, subspace
clustering techniques are certainly gaining momentum as they begin to
be used in fields as diverse as identification and classification of
diseases \cite{montana}, network topology inference \cite{highrankMC},
security and privacy in recommender systems \cite{montanariSC}, system
identification \cite{sysID}, hyper-spectral imaging
\cite{hyperspectral}, identification of switched linear systems
\cite{masysid,ozay}, and music analysis \cite{music} to name just a
few. In spite of all these interesting works, tractable subspace
clustering algorithms either lack a theoretical justification, or are
guaranteed to work under restrictive conditions rarely met in
practice. (We note that although novel and often efficient clustering
techniques come about all the time, establishing rigorous theory for
such techniques has proven to be quite difficult. In the context of
subspace clustering, Section~\ref{comp} offers a partial survey of the
existing literature.) Furthermore, proposed algorithms are not always
computationally tractable. Thus, one important issue is whether
tractable algorithms that can (provably) work in less than ideal
situations---that is, under severe noise conditions and relatively few
samples per subspace---exist.

Elhamifar and Vidal \cite{ehsanSSC} have introduced an approach to
subspace clustering, which relies on ideas from the sparsity and
compressed sensing literature, please see also the longer version
\cite{SSCalg} which was submitted while this manuscript was under
preparation. \emph{Sparse subspace clustering} (SSC)
\cite{ehsanSSC,SSCalg} is computationally efficient since it amounts
to solving a sequence of $\ell_1$ minimization problems and is,
therefore, tractable. Now the methodology in \cite{ehsanSSC} is mainly
geared toward noiseless situations where the points lie exactly on
lower-dimensional planes, and theoretical performance guarantees in
such circumstances are given under restrictive assumptions. Continuing
on this line of work, \cite{ourSSC} showed that good theoretical
performance could be achieved under broad circumstances. However, the
model supporting the theory in \cite{ourSSC} is still noise free.

This paper considers the subspace clustering problem in the presence
of noise. We introduce a tractable clustering algorithm, which is a
natural extension of~SSC, and develop rigorous theory about its
performance; see the results from Section~\ref{mainresults}. In a
nutshell, we propose a statistical mixture model to represent data
lying near a union of subspaces, and prove that in this model, the
algorithm is effective in separating points from different subspaces
as long as there are sufficiently many samples from each subspace and
that the subspaces are not too close to each other. In this theory,
the performance of the algorithm is explained in terms of
interpretable and intuitive parameters such as (1) the values of the
principal angles between subspaces, (2) the number of points per
subspace, (3) the noise level and so on. In terms of these parameters,
our theoretical results indicate that the performance of the algorithm
is in some sense near the limit of what can be achieved by any
algorithm, regardless of tractability.

\subsection{\texorpdfstring{Problem formulation and model.}{Problem formulation and model}}\label{model}

We assume we are given data points lying near a union of unknown
linear subspaces; there are $L$ subspaces $S_1,S_2,\ldots,S_L$ of
$\mathbb{R}^n$ of dimensions $d_1,d_2,\ldots,d_L$. These together
with their
number are completely unknown to us. We are given a point set
$\mathcal{Y}\subset\mathbb{R}^n$ of cardinality $N$, which may be partitioned
as
$\mathcal{Y}=\mathcal{Y}_1\cup\mathcal{Y}_2\cup\cdots\cup
\mathcal{Y}_L$;
for each $\ell\in\{1, \ldots, L\}$, $\mathcal{Y}_\ell$ is a
collection of $N_\ell$ vectors that are ``close'' to subspace
$S_\ell$. The goal is to approximate the underlying subspaces using
the point set $\mathcal{Y}$. One approach is first to assign each data
point to a cluster, and then estimate the subspaces representing each
of the groups with PCA.

Our statistical model assumes that each point $\mathbf{y}\in\mathcal{Y}$ is
of the form
%
%
\begin{equation}
\label{eqmodel} \mathbf{y}= \mathbf{x} + \mathbf{z},
\end{equation}
where $\mathbf{x}$ belongs to one of the subspaces and $\mathbf{z}$ is an independent
stochastic noise term. We suppose that the inverse signal-to-noise
ratio (SNR) defined as $\operatorname{\mathbb{E}}\|\mathbf{z}\|
_2^2/\|
\mathbf{x}\|_{\ell_2}^2$ is
bounded above. Each observation is thus the superposition of a
noiseless sample taken from one of the subspaces and of a stochastic
perturbation whose Euclidean norm is about $\sigma$ times the signal
strength so that $\operatorname{\mathbb{E}}\|z\|_{\ell_2}^2 = \sigma
^2 \|x\|_{\ell_2}^2$.
All the way through, we assume that
%
%
\begin{equation}
\label{eqsimplicity} \sigma< \sigma^\star\quad\mbox{and}\quad\max
_\ell{d_\ell}< c_0 \frac{n}{{(\log N)}^2},
\end{equation}
where $\sigma^\star< 1$ and $c_0$ are fixed numerical constants. To
remove any ambiguity, $\sigma$~is the noise level and $\sigma^\star$
the maximum value it can take on. The second assumption is here to
avoid unnecessarily complicated expressions later on. While more
substantial, the first is not too restrictive since it just says that
the signal $\mathbf{x}$ and the noise $\mathbf{z}$ may have about the same
magnitude. (With an arbitrary perturbation of Euclidean norm equal to
two, one can move from any point $\mathbf{x}$ on the unit sphere to just about
any other point.)

This is arguably the simplest model providing a good starting point
for a theoretical investigation. For the noiseless samples $\mathbf{x}$, we
consider the intuitive \emph{semirandom model} introduced in
\cite{ourSSC}, which assumes that the subspaces are fixed with points
distributed uniformly at random on each subspace. One can think of
this as a mixture model where each component in the mixture is a lower-dimensional subspace. (One can extend the methods to affine subspace
clustering as briefly explained in Section~\ref{secmethod}.) %

\subsection{\texorpdfstring{What makes clustering hard?}{What makes clustering hard}}

Two important parameters fundamentally affect the performance of
subspace clustering algorithms: (1) the distance between subspaces and
(2) the number of samples on each subspace.

\subsubsection{\texorpdfstring{Distance/affinity between subspaces.}{Distance/affinity between subspaces}}

Intuitively, any subspace clustering algorithm operating on noisy data
will have difficulty segmenting observations when the subspaces are
close to each other. We of course need to quantify closeness, and
Definition~\ref{defaffinity} captures a notion of distance or
similarity/affinity between subspaces.
%
%
\begin{definition}
The principal angles $\theta^{(1)},\ldots,\theta^{(d \wedge d')}$
between two subspaces $S$ and $S'$ of dimensions $d$ and $d'$, are
recursively defined by
\[
\cos\bigl(\theta^{(i)}\bigr) = \max_{\mathbf{u}_i\in S}\max
_{\mathbf{v}_i
\in S'} \frac{\mathbf{u}_i^T\mathbf{v}_i}{\llVert \mathbf{u}_i\rrVert
_{\ell
_2}\llVert \mathbf{v}_i\rrVert _{\ell_2}}
\]
with the orthogonality constraints $\mathbf{u}_i^T\mathbf{u}_j = 0$,
$\mathbf{v}_i^T\mathbf{v}_j=0$, $j=1,\ldots,i-1$.
\end{definition}
Alternatively, if the columns of $\mathbf{U}$ and $\mathbf{V}$ are
orthobases for $S$ and $S'$, then the cosine of the principal angles
are the singular values of $\mathbf{U}^T\mathbf{V}$.

%
\begin{definition}
\label{defaffinity}
The normalized affinity between two subspaces is defined by
\[
\operatorname{aff}\bigl(S,S'\bigr)=\sqrt{\frac{\cos^2\theta^{(1)} + \cdots+\cos
^2\theta^{(d
\wedge d')}}{d\wedge d'}}.
\]
\end{definition}
The affinity is a measure of correlation between subspaces. It is low
when the principal angles are nearly right angles (it vanishes when
the two subspaces are orthogonal) and high when the principal angles
are small (it takes on its maximum value equal to one when one
subspace is contained in the other). Hence, when the affinity is high,
clustering is hard whereas it becomes easier as the affinity
decreases. Ideally, we would like our algorithm to be able to handle
higher affinity values---as close as possible to the maximum possible
value.

There is a statistical description of the affinity which goes as
follows: sample independently two unit-normed vectors $\mathbf{x}$ and
$\mathbf{y}$ uniformly at random from $S$~and~$S'$. Then
\[
\operatorname{\mathbb{E}} \bigl\{\bigl(\mathbf{x}^T \mathbf{y}
\bigr)^2 \bigr\} \propto \bigl\{\operatorname{aff}\bigl(S,S'
\bigr) \bigr\}^2,
\]
where the constant of proportionality is $d \vee d'$. Having said
this, there are of course other ways of measuring the affinity between
subspaces; for instance, by taking the cosine of the first principal
angle. We prefer the definition above as it offers the flexibility of
allowing for some principal angles to be small or zero. As an example,
suppose we have a pair of subspaces with a nontrivial
intersection. Then $|\cos\theta^{(1)}| = 1$ regardless of the
dimension of the intersection whereas the value of the affinity would
depend upon this dimension.


\subsubsection{\texorpdfstring{Sampling density.}{Sampling density}}

Another important factor affecting the performance of subspace
clustering algorithms has to do with the distribution of points on
each subspace. In the model we study here, this essentially reduces to
the number of~points that lie on each subspace.\footnote{In a general
deterministic model, where the points have arbitrary orientations on
each subspace, we can imagine that the clustering problem becomes
harder as the points align along an even lower-dimensional
structure.}
%
%
\begin{definition} The sampling density $\rho$ of a subspace is
defined as the number of samples on that subspace per dimension. In
our multi-subspace model, the density of $S_\ell$ is, therefore,
$\rho_\ell=N_\ell/d_\ell$.\footnote{Throughout, we take
$\rho_\ell\le e^{d_\ell/2}$. Our results hold for all other values
by substituting $\rho_\ell$ with $\rho_\ell\wedge e^{d_\ell/2}$
in all the expressions.}
\end{definition}
One expects the clustering problem to become easier as the sampling
density increases. Obviously, if the sampling density of a subspace
$S$ is smaller than one, then any algorithm will fail in identifying
that subspace correctly as there are not sufficiently many points to
identify all the directions spanned by $S$. Hence, we would like a
clustering algorithm to be able to operate at values of the sampling
density as low as possible, that is, as close to one as possible.

\section{\texorpdfstring{Robust subspace clustering: Methods and concepts.}{Robust subspace clustering: Methods and concepts}}
\label{secmethod}

This section introduces our methodology through heuristic arguments
confirmed by numerical \mbox{experiments} while proven theoretical guarantees
about the first step of algorithm follow in Section~\ref{secmain}.
From now on, we arrange the $N$ observed data points as columns of a
matrix $\mathbf{Y}=[\mathbf{y}_1, \ldots, \mathbf{y}_N] \in\mathbb
{R}^{n\times
N}$. With obvious notation, $\mathbf{Y}= \mathbf{X}+ \mathbf{Z}$.

\subsection{\texorpdfstring{The normalized model.}{The normalized model}}

In practice, one may want to normalize the columns of the data matrix
so that for all $i$, $\|\mathbf{y}_i\|_{\ell_2} = 1$ [R-code snippet for
renormalizing a data point $y$ is: \texttt{y <-y/sqrt(sum(y{${}\wedge{}$}2))}]. Since with our SNR assumption, we have
$\|\mathbf{y}\|_{\ell_2} \approx\|\mathbf{x}\|_{\ell_2} \sqrt{1 + \sigma
^2}$ \emph{before} normalization, then \emph{after} normalization:
\[
\mathbf{y}\approx\frac{1}{\sqrt{1 + \sigma^2}} (\mathbf{x}+ \mathbf{z}),
\]
where $\mathbf{x}$ is unit-normed, and $\mathbf{z}$ has i.i.d. random
Gaussian entries
with variance~$\sigma^2/n$.

For ease of presentation, we work---in this section and in the
proofs---with a model $\mathbf{y}= \mathbf{x}+ \mathbf{z}$ in which $\|\mathbf{x}\|
_{\ell_2}=1$
instead of $\|\mathbf{y}\|_{\ell_2} = 1$ (the numerical Section~\ref
{secnumerical} is the exception). The normalized model with
$\|\mathbf{x}\|_{\ell_2}=1$ and $\mathbf{z}$ i.i.d. $\mathcal{N}(0,\sigma
^2/n)$ is
nearly the same as before. In particular, all of our methods and
theoretical results in Section~\ref{secmain} hold with both models in
which either $\|\mathbf{x}\|_{\ell_2} = 1$ or $\|\mathbf{y}\|_{\ell_2} = 1$.

\subsection{\texorpdfstring{The SSC scheme.}{The SSC scheme}}

We describe the approach in \cite{ehsanSSC}, which follows a
three-step procedure:
\begin{longlist}[III.]
\item[I.] Compute a similarity\footnote{We use the terminology
similarity graph or
matrix instead of affinity matrix as not to overload the word
``affinity.''} matrix $\mathbf{W}$ encoding similarities between
sample pairs as to construct a weighted graph~$\mathcal{G}$.
\item[II.] Construct clusters by applying spectral clustering
techniques (e.g.,~\cite{ng2002spectral}) to~$\mathcal{G}$.
\item[III.] Apply PCA to each of the clusters.
\end{longlist}

The novelty in \cite{ehsanSSC} concerns step~I, the construction of
the affinity matrix. Interestingly, similar ideas were introduced
earlier in the statistics literature for the purpose of graphical
model selection \cite{GLASSO}. Now the work \cite{ehsanSSC} of
interest here is mainly concerned with the noiseless situation in
which $\mathbf{Y}= \mathbf{X}$ and the idea is then to express each column
$\mathbf{x}_i$ of $\mathbf{X}$ as a sparse linear combination of all the
other columns. The reason is that under any reasonable condition, one
expects that the \emph{sparsest} representation of $\mathbf{x}_i$ would
only select vectors from the subspace in which $\mathbf{x}_i$ happens to
lie in. Applying the $\ell_1$ norm as the convex surrogate of sparsity
leads to the following sequence of optimization problems:
%
%
\begin{equation}
\label{eql1eq} \min_{\bolds{\beta} \in\mathbb{R}^N} \llVert \bolds {\beta }\rrVert
_{\ell_1}\qquad\mbox{subject to } \mathbf{x}_i=\mathbf{X}
\bolds{\beta}\quad\mbox {and}\quad \beta_i=0.
\end{equation}
Here, $\beta_i$ denotes the $i$th element of $\bolds{\beta}$ and the
constraint $\beta_i=0$ removes the trivial solution that decomposes a
point as a linear combination of itself. Collecting the outcome
of\vadjust{\goodbreak}
these $N$ optimization problems as columns of a matrix $\mathbf{B}$,
\cite{ehsanSSC} sets the $N \times N$ similarity matrix $\mathbf{W}$ to
be ${W}_{ij} = |{B}_{ij}| + |{B}_{ji}|$. [This algorithm clusters
linear subspaces but can also cluster affine subspaces by adding the
constraint $\bolds{\beta}^T\mathbf{1}= 1$ to (\ref{eql1eq}).]

The issue here is that we only have access to the noisy data
$\mathbf{Y}$; that is, we do not see the matrix $\mathbf{X}$ of covariates
but rather a corrupted version $\mathbf{Y}$. This makes the problem
challenging, as unlike conventional sparse recovery problems where
only the response vector $\mathbf{x}_i$ is corrupted, here both the covariates
(columns of $\mathbf{X}$) and the response vector are corrupted. In
particular, it may not be advisable to use (\ref{eql1eq}) with $\mathbf{y}_i$
and $\mathbf{Y}$ in place of $\mathbf{x}_i$ and $\mathbf{X}$ as,
strictly speaking, sparse
representations no longer exist. Observe that the expression
$\mathbf{x}_i=\mathbf{X}\bolds{\beta}$ can be rewritten as
$\mathbf{y}_i=\mathbf{Y}\bolds{\beta}+(\mathbf{z}_i-\mathbf{Z}\bolds
{\beta})$. Viewing
$(\mathbf{z}_i-\mathbf{Z}\bolds{\beta})$ as a perturbation, it is
natural to use ideas
from sparse regression to obtain an estimate $\hat{\bolds{\beta}}$,
which is then used to construct the similarity matrix. In this paper,
we follow the same three-step procedure and shall focus on the first
step in Algorithm~\ref{alg}; that is, on the construction of reliable
similarity measures between pairs of points. Since we have noisy
data, we shall not use (\ref{eql1eq}) here. Also, we add denoising to
step~III, check the output of Algorithm~\ref{alg}. {We would like to
emphasize early on that the theoretical analysis provided in this
paper only concerns the first step---the sparse regression part---of
the algorithm. We do not provide any guarantees for the spectral
clustering step.}

\begin{algorithm}[t]
\caption{Robust SSC procedure}
\begin{algorithmic}
\REQUIRE{A data set $\mathcal{Y}$ arranged as 
columns of $\mathbf{Y}\in\mathbb{R}^{n\times
N}$.}

\STATE1. For each $i \in\{1, \ldots, N\}$, produce a sparse
coefficient sequence $\{\hat{\bolds{\beta}}_i\}$ by regressing the $i$th
vector $\mathbf{y}_i$ onto the other columns of $\mathbf{Y}$. Collect
these as columns of a matrix $\mathbf{B}$.

\STATE2. Form the similarity graph $\mathcal{G}$ with nodes
representing the $N$ data points and edge weights given by $W_{ij} =
|B_{ij}| + |B_{ji}|$.

\STATE3. Sort the eigenvalues $\delta_1 \ge\delta_2 \ge\cdots\ge
\delta_N$ of the normalized Laplacian of $\mathcal{G}$ in descending order,
and set
\[
\hat{L}=N-\mathop{\arg\max}_{i=1,\ldots,N-1} (\delta_i-
\delta_{i+1}).
\]

\STATE4. Apply a spectral clustering technique to the similarity
graph using $\hat{L}$ as the estimated number of clusters to obtain
the partition $\mathcal{Y}_1,\ldots,\mathcal{Y}_{\hat{L}}$.

\STATE5. Use PCA to find the best subspace fits ($\{S_\ell\}_1^L$) to
each of the partitions ($\{\mathcal{Y}_\ell\}_1^L$) and denoise
$\mathbf{Y}$ as to obtain clean data points $\hat{\mathbf{X}}$.

\ENSURE{Subspaces $\{S_\ell\}_1^L$ and cleaned data points
$\hat{\mathbf{X}}$.}
\end{algorithmic}
\label{alg}
\end{algorithm}


\subsection{\texorpdfstring{Performance metrics for similarity measures.}{Performance metrics for similarity measures}}

Given the general structure of the method, we are interested in sparse
regression techniques, which tend to select points in the same
clusters (share the same underlying subspace) over those that do not
share this property. Expressed differently, the hope is that whenever
$B_{ij} \neq0$, $\mathbf{y}_i$ and $\mathbf{y}_j$ {originate from}
the same
subspace. We introduce metrics to quantify performance.

%
%
\begin{definition}[(False discoveries)] Fix $i$ and $j \in\{1, \ldots,
N\}$ and let $\mathbf{B}$ be the outcome of step~1 in Algorithm~\ref{alg}. Then we say that $(i,j)$ obeying $B_{ij} \neq0$ is a
false discovery if $\mathbf{y}_i$ and $\mathbf{y}_j$ do not {originate from}
the same
subspace. 
\end{definition}
%
%
\begin{definition}[(True discoveries)] In the same situation, $(i,j)$
obeying $B_{ij} \neq0$ is a true discovery if $\mathbf{y}_j$ and $\mathbf{y}_i$
{originate from} the same cluster/subspace.
\end{definition}
When there are no false discoveries, we shall say that the \emph{subspace detection property} holds. In this case, the matrix
$\mathbf{B}$ is block diagonal after applying a permutation which makes
sure that columns in the same subspace are contiguous. In some cases,
the sparse regression method may select vectors from other subspaces
and this property will not hold. However, it might still be possible
to detect and construct reliable clusters by applying steps 2--5 in
Algorithm~\ref{alg}.

\subsection{\texorpdfstring{LASSO with data-driven regularization.}{LASSO with data-driven regularization}}

A natural sparse regression strategy is the LASSO:
%
%
\begin{equation}
\label{eqlasso} \min_{\bolds{\beta} \in\mathbb{R}^N} \frac{1}{2}\llVert \mathbf
{y}_i-\mathbf{Y}\bolds{\beta}\rrVert _{\ell_2}^2+
\lambda\llVert \bolds{\beta}\rrVert _{\ell_1}\qquad\mbox{subject to }
\beta_i=0.
\end{equation}
Whether such a methodology should succeed is unclear as we are not
under a traditional model for both the response $\mathbf{y}_i$ and the
covariates $\mathbf{Y}$ are noisy; see \cite{MUS} for a discussion of sparse
regression under matrix uncertainty and what can go wrong. The main
contribution of this paper is to show that if one selects $\lambda$ in
a data-driven fashion, then compelling practical and theoretical
performance can be achieved.




\subsubsection{\texorpdfstring{About as many true discoveries as dimension.}{About as many true discoveries as dimension}}

The nature of the problem is such that we wish to make few false
discoveries (and not link too many pairs belonging to different
subspaces) and so we would like to choose $\lambda$ large. At the same
time, we wish to make many true discoveries, whence a natural trade
off. The reason why we need many true discoveries is that spectral
clustering needs to assign points to the same cluster when they indeed
lie near the same subspace. If the matrix $\mathbf{B}$ is too sparse,
this will not happen.

We now introduce a principle for selecting the regularization
parameter; our exposition here is informal and we refer to\vadjust{\goodbreak} Section~\ref
{secmain} and the supplemental article \cite{RSCsupp} for precise
statements and
proofs. Suppose we have noiseless data so that $\mathbf{Y}= \mathbf
{X}$, and thus
solve (\ref{eql1eq}) with equality constraints. Under our model,
assuming there are no false discoveries, the optimal solution is
guaranteed to have exactly $d$---the dimension of the subspace the
sample under study belongs to---nonzero coefficients with probability
one. That is to say, when the point lies in a $d$-dimensional space,
we find $d$ ``neighbors.''

The selection rule we shall analyze in this paper is to take $\lambda$
as large as possible (as to prevent false discoveries) while making
sure that the number of true discoveries is also on the order of the
dimension $d$, typically in the range $[0.5 d, 0.8 d]$. We can say
this differently. Imagine that all the points lie in the same subspace
of dimension $d$ so that every discovery is true. Then we wish to
select $\lambda$ in such a way that the number of discoveries is a
significant fraction of $d$, the number one would get with noiseless
data. Which value of $\lambda$ achieves this goal? We will see in
Section~\ref{seclambda} that the answer is around $1/\sqrt{d}$. To
put this in context, this means that we wish to select a
regularization parameter which depends upon the dimension $d$ of the
subspace our point comes from. (We are aware that the dependence on
$d$ is unusual as in sparse regression the regularization parameter
usually does not depend upon the sparsity of the solution.) In turn,
this immediately raises another question: since $d$ is unknown, how
can we proceed? In Section~\ref{sectwo-step}, we will see that it is
possible to guess the dimension and construct fairly reliable
estimates.

\subsubsection{\texorpdfstring{Data-dependent regularization.}{Data-dependent regularization}}
\label{seclambda}

We now discuss values of $\lambda$ obeying the demands formulated in
the previous section. Our arguments are informal and we refer the
reader to Section~\ref{secmain} for rigorous statements and to
the supplemental article \cite{RSCsupp}. First, it simplifies the
discussion to assume that we have no noise (the noisy case assuming
$\sigma\ll1$ is similar). Following our earlier discussion, imagine
we have a vector $\mathbf{x}\in\mathbb{R}^n$ lying in the
$d$-dimensional span of
the columns of an $n \times N$ matrix $\mathbf{X}$. We are interested in
values of $\lambda$ so that the minimizer $\hat{\bolds{\beta}}$ of
the LASSO
functional
\[
K(\bolds{\beta},\lambda) = \tfrac{1}{2} \llVert \mathbf{x}-\mathbf {X}\bolds{
\beta}\rrVert _{\ell_2}^2+\lambda\llVert \bolds{\beta }\rrVert
_{\ell_1}
\]
has a number of nonzero components in the range $[0.5 d, 0.8 d]$,
say. Now let $\hat{\bolds{\beta}}_{\mathrm{eq}}$ be the solution of
the problem with equality
constraints, or equivalently of the problem above with $\lambda
\rightarrow0^+$. Then
%
%
\begin{equation}
\label{eqscaling} \tfrac{1}{2} \llVert \mathbf{x}-\mathbf{X}\hat{\bolds{\beta}}
\rrVert _{\ell
_2}^2 \le K(\hat{\bolds{\beta}},\lambda) \le K(
\hat{\bolds{\beta }}_{\mathrm
{eq}},\lambda) = \lambda \llVert \hat{\bolds{\beta}}_{\mathrm{eq}}\rrVert _{\ell_1}.
\end{equation}
We make two observations: the first is that if $\hat{\bolds{\beta}}$
has a number
of nonzero components in the range $[0.5 d, 0.8 d]$, then $\|\mathbf
{x}-\mathbf{X}
\hat{\bolds{\beta}}\|_{\ell_2}^2$ has to be greater than or equal to
a fixed
numerical constant. The reason is that we cannot approximate to
arbitrary accuracy a generic vector living in a $d$-dimensional
subspace as a linear combination of about $d/2$ elements from that
subspace. The second observation is that $\|\hat{\bolds{\beta}}_{\mathrm{eq}}\|_{\ell_1}$
is on the order of $\sqrt{d}$, which is a fairly intuitive scaling (we
have $d$ coordinates, each of size about $1/\sqrt{d}$). This holds
with the proviso that the algorithm operates correctly in the
noiseless setting and does not select columns from other
subspaces. Then (\ref{eqscaling}) implies that $\lambda$ has to scale
at least like $1/\sqrt{d}$. On the other hand, $\hat{\bolds{\beta}}=
\mathbf{0}$ if
$\lambda\ge\|\mathbf{X}^T \mathbf{x}\|_{\ell_\infty}$. Now the informed
reader knows
that $\|\mathbf{X}^T \mathbf{x}\|_{\ell_\infty}$ scales at most like
$\sqrt {(\log
N)/d}$ so that choosing $\lambda$ around this value yields no
discovery (one can refine this argument to show that $\lambda$ cannot
be higher than a constant times $1/\sqrt{d}$ as we would otherwise
have a solution that is too sparse). Hence, $\lambda$ is around
$1/\sqrt{d}$.

It might be possible to compute a precise relationship between
$\lambda$ and the expected number of true discoveries in an asymptotic
regime in which the number of points and the dimension of the subspace
both increase to infinity in a fixed ratio by adapting ideas from
\cite{DMP,BMLasso}. We will not do so here as this is beyond the scope
of this paper. Rather, we investigate this relationship by means of a
numerical study.

%
\begin{figure}

\includegraphics{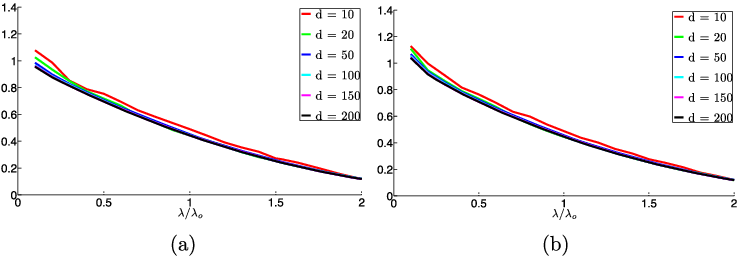}

\caption{Average number of true discoveries normalized by subspace
dimension for values of $\lambda$ in an interval including the
heuristic $\lambda_o= 1/\sqrt{d}$. \textup{(a)} $\sigma=0.25$. \textup{(b)} $\sigma=0.5$.}\label{figMT}
\end{figure}

Here, we fix a single subspace in $\mathbb{R}^{n}$ with $n ={}$2000.
We use a
sampling density equal to $\rho=5$ and vary the dimension $d \in\{10,
20, 50, 100, 150, 200\}$ of the subspace as well as the noise level
$\sigma\in\{ 0.25, 0.5\}$. For each data point, we solve
(\ref{eqlasso}) for different values of $\lambda$ around the
heuristic $\lambda_o = 1/\sqrt{d}$, namely, $\lambda\in[0.1
\lambda_o, 2\lambda_o]$. In our experiments, we declare a discovery if
an entry in the optimal solution exceeds $10^{-3}$. Figure~\ref{figMT}(a)~and~(b) shows the number of discoveries per
subspace dimension (the number of discoveries divided by $d$). One can
clearly see that the curves corresponding to various subspace
dimensions stack up on top of each other, thereby confirming that a
value of $\lambda$ on the order of $1/\sqrt{d}$ yields a fixed
fraction of true discoveries. Further inspection also reveals that the
fraction of true discoveries is around $50\%$ near
$\lambda=\lambda_o$, and around $75\%$ near $\lambda=\lambda_o/2$. We
have observed empirically that increasing $\rho$ typically yields a
slight increase in the fraction of true discoveries (unless, of
course, $\rho$ is exponentially large in $d$).

\subsubsection{\texorpdfstring{The false-true discovery trade off.}{The false-true discovery trade off}}

We now show empirically that in our model choosing $\lambda$ around
$1/\sqrt{d}$ typically yields very few false discoveries as well as
many true discoveries; this holds with the proviso that the subspaces
are of course not very close to each other.

In this simulation, $22$ subspaces of varying dimensions in $\mathbb{R}^{n}$
with $n = {}$2000 have been independently selected uniformly at random;
there are $5$, $4$, $3$, $4$, $4$~and~$2$ subspaces of respective
dimensions $200$, $150$, $100$, $50$, $20$ and $10$. This is a
challenging regime since the sum of the subspace dimensions equals
2200 and exceeds the ambient dimension (the clean data matrix
$\mathbf{X}$ has full rank). We use a sampling density equal to $\rho=5$
for each subspace and set the noise level to $\sigma= 0.3$. To
evaluate the performance of the optimization problem (\ref{eqlasso}),
we proceed by selecting a subset of columns as follows: for each
dimension, we take $100$ cases at random belonging to subspaces of
that dimension. Hence, the total number of test cases is $m=600$ so
that we only solve $m$ optimization problems (\ref{eqlasso}) out of
the total $N$ possible cases. Below, $\bolds{\beta}^{(i)}$ is the
solution to (\ref{eqlasso}) and $\bolds{\beta}^{(i)}_S$ its restriction
to columns with indices in the same subspace. Hence, a nonvanishing
entry in $\bolds{\beta}^{(i)}_S$ is a true discovery, and likewise, a
nonvanishing entry in $\bolds{\beta}^{(i)}_{S^c}$ is false. For each
data point,\vspace*{1pt} we sweep the tuning parameter $\lambda$ in (\ref{eqlasso})
around the heuristic $\lambda_o = 1/\sqrt{d}$ and work with $\lambda
\in[0.05 \lambda_o, 2.5\lambda_o]$. In our experiments, a discovery
is a value obeying $|B_{ij}| > 10^{-3}$.

%
\begin{figure}

\includegraphics{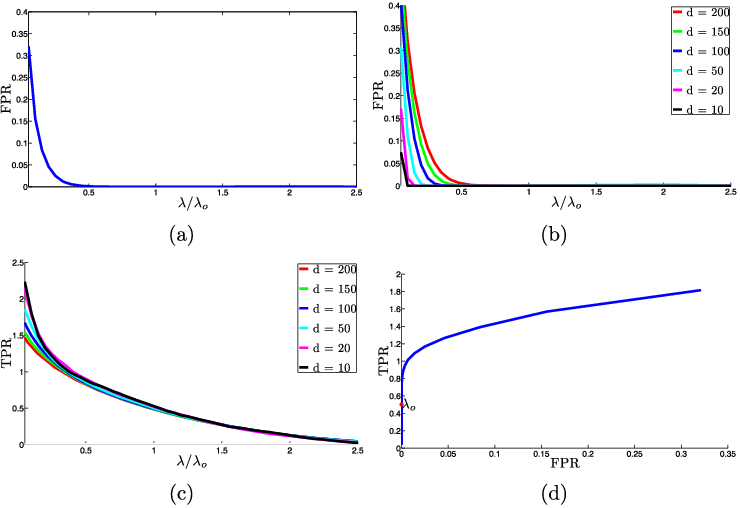}

\caption{Performance of LASSO for values of $\lambda$ in an
interval including the heuristic $\lambda_o= 1/\sqrt{d}$.
\textup{(a)}~Average number of false discoveries normalized by
$(n-d)$ (FPR) on all $m$ sampled data points. \textup{(b)}~FPR for
different subspace dimensions. Each curve represents the
average FPR over those samples originating from subspaces of
the same dimension. \textup{(c)} Average number of true discoveries
per dimension for various dimensions (TPR). \textup{(d)} TPR vs. FPR
(ROC curve). The point corresponding to $\lambda=\lambda_o$
is marked as a red dot.}\label{figFPR}
\end{figure}

In analogy with the signal detection literature, we view the empirical
averages of $\|\bolds{\beta}^{(i)}_{S^c}\|_{\ell_0}/(n-d)$ and
$\|\bolds{\beta}^{(i)}_S\|_{\ell_0}/d$ as False Positive Rate (FPR) and
True Positive Rate (TPR). On the one hand, Figure~\ref{figFPR}(a)~and~(b) shows that for values around \mbox{$\lambda=\lambda_o$}, the
FPR is zero (so there are no false discoveries). On the other hand,
Figure~\ref{figFPR}(c) shows that the TPR curves corresponding to
different dimensions are very close to each other and resemble those
in Figure~\ref{figFPR}(c) in which all the points belong to the same
cluster with no opportunity of making a false discovery. Hence, taking
$\lambda$ near $1/\sqrt{d}$ gives a performance close to what can be
achieved in a noiseless situation. That is to say, we have no false
discovery and a number of true discoveries about $d/2$ if we choose
$\lambda= \lambda_o$. Figure~\ref{figFPR}(d) plots TPR versus FPR
[a.k.a. the Receiver Operating Characteristic (ROC) curve] and
indicates that $\lambda= \lambda_o$ (marked by a red dot) is an
attractive trade-off as it provides no false discoveries and
sufficiently many true discoveries.

\subsubsection{\texorpdfstring{A two-step procedure.}{A two-step procedure}}
\label{sectwo-step}

Returning to the selection of the regularization parameter, we would
like to use $\lambda$ on the order of $1/\sqrt{d}$. However, we do not
know $d$ and proceed by substituting an estimate. In the next
section, we will see that we are able to quantify theoretically the
performance of the following proposal: (1)~run a hard constrained
version of the LASSO and use an estimate $\hat{d}$ of dimension based
on the $\ell_1$ norm of the fitted coefficient sequence; (2)~impute a
value for $\lambda$ constructed from $\hat{d}$. The two-step procedure
is explained in Algorithm~\ref{mod}. Again, our exposition is informal
here and we refer to Section~\ref{secmain} for precise statements.
\begin{algorithm}[t]
\caption{Two-step procedure with data-driven regularization}
\begin{algorithmic}
\FOR{$i=1, \ldots, N$}
\STATE1. Solve
%
%
\begin{equation}
\label{eqLASSO1} \bolds{\beta}^\star= \arg\min_{\bolds{\beta} \in\mathbb
{R}^N}
\llVert \bolds{\beta}\rrVert _{\ell_1}\qquad\mbox{subject to } \llVert
\mathbf{y}_i-\mathbf{Y}\bolds{\bolds{\beta}}\rrVert _{\ell
_2}\le
\tau\quad\mbox{and}\quad\bolds{\beta}_i =0.
\end{equation}
\STATE2. Set $\lambda= f(\|\bolds{\beta}^\star\|_{\ell_1})$.
\STATE
3. Solve
\[
\hat{\bolds{\beta}}=\arg\min_{\bolds{\beta} \in\mathbb{R}^N} \frac
{1}{2}\llVert
\mathbf{y}_i-\mathbf{Y}\bolds{\beta}\rrVert _{\ell
_2}^2
+ \lambda \llVert \bolds{\beta}\rrVert _{\ell_1}\qquad\mbox{subject to }
\bolds{\beta}_i=0.
\]
\STATE4. Set $\mathbf{B}_i=\hat{\bolds{\beta}}$.
\ENDFOR
\end{algorithmic}
\label{mod}
\end{algorithm}

To understand the rationale behind this, imagine we have noiseless
data---that is, $\mathbf{Y}= \mathbf{X}$---and are solving (\ref
{eql1eq}), which simply
is our first step (\ref{eqLASSO1}) with the proviso that $\tau=
0$. When there are no false discoveries, one can show that the
$\ell_1$ norm of $\bolds{\beta}^\star$ is roughly of size $\sqrt {d}$ as
shown in Lemma~A.2 from the supplemental article \cite{RSCsupp}. This
suggests using a multiple of $\|\bolds{\beta}^\star\|_{\ell_1}$ as a
proxy for $\sqrt{d}$. To drive this point home, take a look at
Figure~\ref{figl1norm}(a) which solves (\ref{eqLASSO1}) with the same
data as
in the previous example and $\tau= 2\sigma$. The plot reveals that
the values of $\|\bolds{\beta}^\star\|_{\ell_1}$ fluctuate around
$\sqrt{d}$. This is shown more clearly in Figure~\ref{figl1norm}(b), which shows that
$\|\bolds{\beta}^\star\|_{\ell_1}$ is concentrated around
$\frac{1}{4}\sqrt{d}$ with, as expected, higher volatility at lower
values of dimension.
%
%
\begin{figure}[b]

\includegraphics{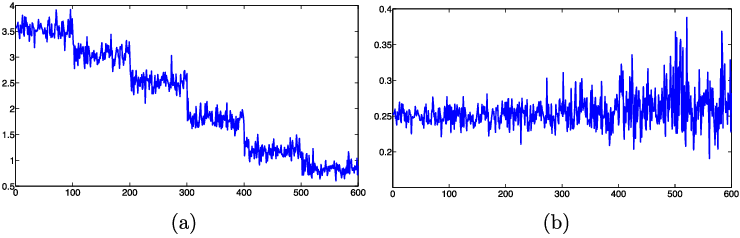}

\caption{Optimal values of (\protect\ref{eqLASSO1}) for $600$ samples
using $\tau=2\sigma$. The first $100$ values correspond to
points originating from subspaces of dimension $d = 200$,
the next $100$ from those of dimension $d = 150$, and so on
through $d \in\{100, 50, 20, 10\}$. \textup{(a)} Value of
$\llVert \bolds{\beta}^*\rrVert _{\ell_1}$. \textup{(b)} Value of
$\llVert \bolds{\beta}^*\rrVert _{\ell_1}/\sqrt{d}$.}\label{figl1norm}
\end{figure}

Under suitable assumptions, we shall see in Section~\ref{secmain}
that with noisy data, there are simple rules for selecting $\tau$ that
guarantee, with high probability, that there are no false
discoveries. To be concrete, one can take $\tau= 2\sigma$ and $f(t)
\propto t^{-1}$. Returning to our running example, we have
$\|\bolds{\beta}^\star\|_{\ell_1}\approx\frac{1}{4}\sqrt{d}$. Plugging
this into $\lambda=1/\sqrt{d}$ suggests taking $f(t) \approx0.25
t^{-1}$. The plots in Figure~\ref{figtwostep} demonstrate that this
is indeed effective. Experiments in Section~\ref{secnumerical}
indicate that this is a good choice on real data as well.

%
\begin{figure}

\includegraphics{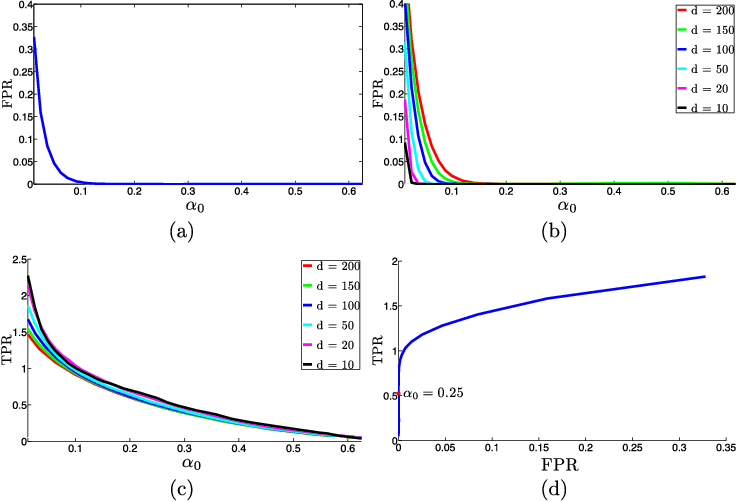}

\caption{Performance of the two-step procedure using
$\tau=2\sigma$ and $f(t) = \alpha_0t^{-1}$ for values of
$\alpha_0$ around the heuristic $\alpha_0=0.25$. \textup{(a)}~False
positive rate (FPR). \textup{(b)}~FPR for various subspace
dimensions. \textup{(c)}~True positive rate (TPR). \textup{(d)}~TPR
vs. FPR.}
\label{figtwostep}
\end{figure}

The two-step procedure requires solving two LASSO problems for each
data point and is useful when there are subspaces of large dimensions
(in the hundreds, say) and some others of low-dimensions (three or
four, say). In some applications such as motion segmentation in
computer vision, the dimensions of the subspaces are all equal and
known in advance \cite{tomasi1992shape}. In this case, one can forgo
the two-step procedure and
simply set $\lambda= 1/\sqrt{d}$.

\section{\texorpdfstring{Theoretical results.}{Theoretical results}}\label{secmain}

This section presents our main theoretical results concerning the
performance of the two-step procedure (Algorithm~\ref{mod}). We defer
the proof of these results to the supplemental article \cite{RSCsupp}.
We make two assumptions:
\begin{itemize}
\item\textit{Affinity condition}. We say that a subspace $S_\ell$ obeys
the \emph{affinity condition} if
%
%
\begin{equation}
\label{eqaffinitycond} \max_{k\dvtx    k \neq\ell} \operatorname{aff}(S_\ell,
S_k) \le{\kappa_0}/{\log N},
\end{equation}
where $\kappa_0$ a fixed numerical constant.

\item\textit{Sampling condition}. We say that subspace $S_\ell$ obeys
the \emph{sampling condition} if
%
%
\begin{equation}
\label{eqsamplingcond} \rho_\ell\ge\rho^\star,
\end{equation}
where $\rho^\star$ is a fixed numerical constant.
\end{itemize}
The careful reader might argue that we should require smaller affinity
values as the noise level increases. The reason why $\sigma$ does not
appear in (\ref{eqaffinitycond}) is that we assumed a bounded noise
level. For higher values of $\sigma$, the affinity condition would
read as in (\ref{eqaffinitycond}) with a right-hand side equal to
\[
\kappa= \frac{\kappa_0}{\log
N}-\sigma\sqrt{\frac{d_\ell}{2n\log N}}.
\]

\subsection{\texorpdfstring{Main results.}{Main results}}\label{mainresults}
From here on, we use $d(i)$ to refer to the dimension of the subspace
the vector $\mathbf{y}_i$ originates from. $N(i)$ and $\rho(i)$ are used
in a
similar fashion for the number and density of points on this subspace.

%
\begin{theorem}[(No false discoveries)]
\label{nofwithd}
Assume that the subspace attached to the $i$th column obeys the
affinity and sampling conditions and that the noise level $\sigma$ is
bounded as in (\ref{eqsimplicity}), where $\sigma^\star$ is a
sufficiently small numerical constant. In Algorithm~\ref{mod}, take
$\tau= 2\sigma$ and $f(t)$ obeying $f(t) \ge0.707 \sigma t^{-1}$.
Then with high probability,\footnote{Probability at least
$1-2e^{-\gamma_1n}-6e^{-\gamma_2d(i)}-e^{-\sqrt{N(i)d(i)}}-\frac{23}{N^2}$,
for fixed numerical constants $\gamma_1$, $\gamma_2$.} there is no
false discovery in the $i$th column of~$\mathbf{B}$.
\end{theorem}
%
%
\begin{theorem}[(Many true discoveries)]
\label{mtwithd}
Consider the same setup as in Theorem~\ref{nofwithd} with $f(\cdot)$ also
obeying $f(t) \le\alpha_0 t^{-1}$ for some numerical constant $\alpha
_0$. Then
with high probability,\footnote{Probability at least
$1-2e^{-\gamma_1n}-6e^{-\gamma_2d(i)}-e^{-\sqrt{N(i)d(i)}}-\frac{23}{N^2}$,
for fixed numerical constants $\gamma_1$, $\gamma_2$.} there are at
least
%
%
\begin{equation}
\label{numtrue} c_1 \frac{d(i)}{\log\rho(i)}
\end{equation}
true discoveries in the $i$th column ($c_1$ is a positive numerical
constant).
\end{theorem}

The above results indicate that the {first step of the} algorithm
works correctly in fairly broad conditions. To give an example, assume
two subspaces of dimension $d$ overlap in a smaller subspace of
dimension $s$ but are orthogonal to each other in the remaining
directions (equivalently, the first $s$ principal angles are $0$ and
the rest are $\pi/2$). In this case, the affinity between the two
subspaces is equal to $\sqrt{s/d}$ and (\ref{eqaffinitycond}) allows
$s$ to grow almost linearly in the dimension of the subspaces. Hence,
subspaces can have intersections of large dimensions. In contrast,
previous work with perfectly noiseless data
\cite{elhamifar2010clustering} would impose to have a first principal
angle obeying $|\cos\theta^{(1)}| \le{1}/{\sqrt{d}}$ so that the
subspaces are practically orthogonal to each other. Whereas our result
shows that we can have an average of the cosines practically constant,
the condition in \cite{elhamifar2010clustering} asks that the maximum
cosine be very small.\looseness=1

In the noiseless case, \cite{ourSSC} showed that when the sampling
condition holds and
\[
\max_{k\dvtx    k \neq\ell}\operatorname{aff}(S_\ell, S_k)
\le \kappa _0\frac{\sqrt{\log\rho_\ell}}{\log N}
\]
(albeit with slightly different values $\kappa_0$ and $\rho^\star$),
then applying the noiseless version~(\ref{eql1eq}) of the algorithm
also yields no false discoveries. Hence, with the proviso that the
noise level is not too large, conditions under which the algorithm is
provably correct are essentially the same.

Earlier, we argued that we would like to have, if possible, an
algorithm provably working at (1) high values of the affinity
parameters and (2) low values of the sampling density as these are the
conditions under which the clustering problem is challenging. (Another
property on the wish list is the ability to operate properly with high
noise or low SNR and this is discussed next.) In this context, since
the affinity is at most one, our results state that the affinity can
be within a log factor from this maximum possible value. The number of
samples needed per subspace is minimal as well. That is, as long as
the density of points on each subspace is larger than a constant
$\rho>\rho^\star$, the algorithm succeeds.\footnote{This is with the
proviso that the density does not grow exponentially in the
dimension of the subspace. This is not a restrictive assumption as
having exponentially many points from the same subspace makes the
problem especially easy.}

We would like to have a procedure capable of making no false
discoveries and many true discoveries at the same time. Now in the
noiseless case, whenever there are no false discoveries, the $i$th
column contains exactly $d(i)$ true discoveries. Theorem~\ref{mtwithd}
states that as long as the noise level $\sigma$ is less than a fixed
numerical constant, the number of true discoveries is roughly on the
same order as in the noiseless case. In other words, a noise level of
this magnitude does not fundamentally affect the performance of the
algorithm. This holds even when there is great variation in the
dimensions of the subspaces, and is possible because $\lambda$ is
appropriately tuned in an adaptive fashion.

The number of true discoveries is shown to scale at least like
dimension over the log of the density. This may suggest that the
number of true discoveries decreases (albeit very slowly) as the
sampling density increases. This behavior is to be expected: when the
sampling density becomes exponentially large (in terms of the
dimension of the subspace) the number of true discoveries become small
since we need fewer columns to synthesize a point. In fact, the
$d/\log\rho$ behavior seems to be the correct scaling. Indeed, when
the density is low and $\rho$ takes on a small value, (\ref{numtrue})
asserts that we make on the order of $d$ discoveries, which is
tight. Imagine now that we are in the high-density regime and $\rho$
is exponential in~$d$. Then as the points gets tightly packed, we
expect to have only one discovery in accordance with (\ref{numtrue}).

Theorem~\ref{mtwithd} establishes that there are many true
discoveries. This would not be useful for clustering purposes if there
were only a handful of very large true discoveries and all the others
of negligible magnitude. The reason is that the similarity matrix
$\mathbf{W}$ would then be close to a sparse matrix and we would run the
risk of splitting true clusters. Our proofs show that this does not
happen although we do not present an argument for lack of space.
Rather, we demonstrate this property empirically. On our running
example, Figure~\ref{fighist1}(a) and~(b) shows that the
histograms of appropriately normalized true discovery values resemble
a bell-shaped curve. Note that each true discovery corresponds to a
nonzero coefficient which can take on either a positive or negative
value.
%
%
\begin{figure}

\includegraphics{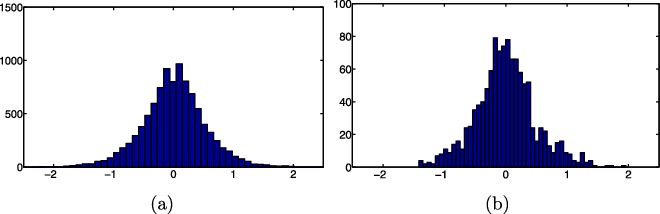}

\caption{Histograms of the true discovery values from the two step
procedure with $\alpha_0=0.25$ (multiplied by $\sqrt{d}$). \textup{(a)}~$d=200$.
\textup{(b)}~$d=20$.}\label{fighist1}
\end{figure}

As stated numerous times, our theoretical analysis only concerns the
first step of the algorithm. We now wish to explain how these
theoretical results relate to complete guarantees for clustering.
First, Theorem~\ref{nofwithd} states that clusters that should be
disconnected from each other are, in fact, disconnected so that the
algorithm does not group together points from different subspaces. To
guarantee perfect clustering, it is then sufficient to show that each
restriction of the similarity graph to a subspace is connected. Due to
the nature of the random model under study, a~subgraph resembles an
Erd\H os--R\`eyni graph with the probability of having an edge roughly
proportional to the number of true discoveries. As long as there are
sufficiently many true discoveries (as shown in Theorem
\ref{mtwithd}), such a graph is well connected---in fact, it has very
good expansion properties. Proving that each subgraph is indeed
connected is a problem we regard as interesting, the main challenge
being caused by the dependencies the algorithm generates. Second, a
more quantitative characterization of the expansion or connectedness
of each subgraph via Cheeger's constant or the eigenvalue gap may
ultimately demonstrate that the algorithm succeeds even in the
presence of few false discoveries with small values of $W_{ij}$;
please see \cite{kannan2009spectral} and references therein.

Finally, we would like to comment on the fact that our main results
hold when $\lambda$ belongs to a fairly broad range of values. First,
when all the subspaces have small dimensions, one can choose the same
value of $\lambda$ for all the data points since $1/\sqrt{d}$ is
essentially constant. Hence, when we know a priori that we are in such
a situation, there may be no need for the two-step procedure. (We
would still recommend the conservative two-step procedure because of
its superior empirical performance on real data.) Second, the proofs
also reveal that if we have knowledge of the dimension of the largest
subspace $d_{\max}$, the first theorem holds with a fixed value
of $\lambda$ proportional to $\sigma/\sqrt{d_{\max}}$. Third,
when the subspaces themselves are drawn at random, the first theorem
holds with a fixed value of $\lambda$ proportional to $\sigma(\log
N)/\sqrt{n}$. (Both these statements follow by plugging these values
of $\lambda$ in the proofs of the supplemental article \cite{RSCsupp} and we omit the
calculations.) We merely mention these variants to give a sense of
what our theorems can also give. As explained earlier, we recommend
the more conservative two-step procedure with the proxy for
$1/\sqrt{d}$. The reason is that using a higher value of $\lambda$
allows for a larger value of $\kappa_0$ in (\ref{eqaffinitycond}),
which says that the subspaces can be even closer. In other words, we
can function in a more challenging regime. To drive this point home,
consider the noiseless problem. When the subspaces are close, the
equality constrained $\ell_1$ problem may yield some false
discoveries. However, if we use the LASSO version---even though the
data is noiseless---we may end up with no false discoveries while
maintaining sufficiently many true discoveries.

\section{\texorpdfstring{The bias-corrected Dantzig selector.}{The bias-corrected Dantzig selector}}\label{dantzigsection}

One can think of other ways of performing the first step in Algorithm
\ref{alg} and this section discusses another approach based on a
modification of the Dantzig selector, a popular sparse regression
technique \cite{dantzig}. Unlike the two-step procedure, we do not
claim any theoretical guarantees for this method and shall only
explore its properties on real and simulated data.

Applied directly to our
problem, the Dantzig selector takes the form
%
%
\begin{equation}
\label{eqds} \qquad\min_{\bolds{\beta} \in\mathbb{R}^N} \llVert \bolds {\beta }\rrVert
_{\ell_1}\qquad\mbox{subject to } \bigl\|\mathbf{Y}_{(-i)}^T(
\mathbf{y}_i-\mathbf{Y}\bolds {\bolds{\beta}})\bigr\|_{\ell
_\infty} \le
\lambda \quad\mbox{and}\quad\bolds{\beta}_i =0,
\end{equation}
where $\mathbf{Y}_{(-i)}$ is $\mathbf{Y}$ with the $i$th column
deleted. However, this
is hardly suitable since the design matrix $\mathbf{Y}$ is
corrupted. Interestingly, recent work \cite{MUS,ImprovedMUS} has
studied the problem of estimating a sparse vector from the standard
linear model under uncertainty in the design matrix. The setup in
these papers is close to our problem and we propose a modified Dantzig
selection procedure inspired but not identical to the methods set
forth in \cite{MUS,ImprovedMUS}.

\subsection{\texorpdfstring{The correction.}{The correction}}

If we had clean data, we would solve
(\ref{eql1eq}); this is (\ref{eqds}) with $\mathbf{Y}= \mathbf
{X}$ and $\lambda=
0$. Let $\bolds{\beta}^I$ be the solution to this ideal noiseless
problem. Applied to our problem, the main idea in
\cite{MUS,ImprovedMUS} would be to find a formulation that resembles
(\ref{eqds}) with the property that $\bolds{\beta}^I$ is
feasible. Since $\mathbf{x}_i = \mathbf{X}_{(-i)} {\bolds{\beta}^I}_{(-i)}$,
observe that we have the
following decomposition:
\begin{eqnarray*}
\mathbf{Y}_{(-i)}^T\bigl(\mathbf{y}_i - \mathbf{Y}
\bolds{\beta}^I\bigr) & =& (\mathbf{X}_{(-i)} + \mathbf{Z}
_{(-i)})^T\bigl(\mathbf{z}_i - \mathbf{Z}\bolds{
\beta}^I \bigr)
\\
& =& \mathbf{X}_{(-i)}^T\bigl(\mathbf{z}_i-
\mathbf{Z}\bolds{\beta}^I\bigr) + \mathbf{Z}_{(-i)}^T
\mathbf{z}_i - \mathbf{Z}_{(-i)}^T \mathbf{Z}\bolds{
\beta}^I.
\end{eqnarray*}
Then the conditional mean is given by
\[
\operatorname{\mathbb{E}}\bigl[\mathbf{Y}_{(-i)}^T\bigl(
\mathbf{y}_i - \mathbf {Y}\bolds{\beta }^I\bigr) |
\mathbf{X}\bigr] = - \operatorname{\mathbb {E}}\mathbf{Z}_{(-i)}^T
\mathbf{Z}_{(-i)} {\bolds{\beta}^I}_{(-i)} = -
\sigma^2 {\bolds{\beta}^I}_{(-i)}.
\]
In other words,
\[
\sigma^2 {\bolds{\beta}^I}_{(-i)} +
\mathbf{Y}_{(-i)}^T\bigl(\mathbf{y}_i - \mathbf{Y}
\bolds{\beta}^I\bigr) = \bolds{\xi},
\]
where $\bolds{\xi}$ has mean zero. In Section~\ref{secvariance}, we
compute the variance of the $j$th component $\xi_j$, given by
%
%
\begin{equation}
\label{eqvariance} \operatorname{\mathbb{E}}\xi_j^2 =
\frac{\sigma^2}{n}\bigl(1 + \bigl\| \bolds{\beta}^I\bigr\|_{\ell_2}^2
\bigr) + \frac{\sigma
^4}{n} \bigl(1 + \bigl(\beta^I_j
\bigr)^2 + \bigl\|\bolds{\beta}^I\bigr\|^2_{\ell
_2}
\bigr).
\end{equation}
Owing to our Gaussian assumptions, $|\xi_j|$ shall be smaller than 3
or 4 times this standard deviation, say, with high probability.

Hence, we may want to consider a procedure of the form
%
%
\begin{eqnarray}
\label{eqcds} \qquad&&\min_{\bolds{\beta} \in\mathbb{R}^N} \llVert \bolds {\beta }\rrVert
_{\ell_1}\qquad\mbox {subject to } \bigl\|\mathbf{Y}_{(-i)}^T(
\mathbf{y}_i-\mathbf{Y}\bolds {\bolds{\beta}}) + \sigma^2
\bolds{\beta}_{(-i)}\bigr\|_{\ell_\infty} \le\lambda\quad \mbox{and}
\nonumber\\[-8pt]\\[-8pt]
&& \bolds{\beta}_i =0.\nonumber
\end{eqnarray}
It follows that if we take $\lambda$ to be a reasonable multiple of
(\ref{eqvariance}), then $\bolds{\beta}^I$ would obey the
constraint in
(\ref{eqcds}) with high probability. Hence, we would need to
approximate the variance (\ref{eqvariance}). Numerical simulations
together with asymptotic calculations presented in the supplemental article \cite{RSCsupp} give that $\|\bolds{\beta}^I\|_{\ell_2} \le1$ with
very high
probability. Thus, neglecting the term in $(\beta^I_j)^2$,
\[
\operatorname{\mathbb{E}}\xi_j^2 \approx
\frac{\sigma^2}{n}\bigl(1 + \sigma^2\bigr) \bigl(1 + \bigl\|\bolds{
\beta}^I\bigr\|_{\ell_2}^2\bigr) \le2 \frac{\sigma^2}{n}
\bigl(1 + \sigma^2\bigr).
\]
This suggests taking $\lambda$ to be a multiple of $\sqrt{2/n}
\sigma  \sqrt{1 + \sigma^2}$. This is interesting because the
parameter $\lambda$ does not depend on the dimension of the underlying
subspace. We shall refer to (\ref{eqcds}) as the \emph{bias-corrected
Dantzig selector}, which resembles the proposal in
\cite{MUS,ImprovedMUS} for which the constraint is a bit more
complicated and of the form $\|\mathbf{Y}_{(-i)}^T(\mathbf{y}_i-\mathbf
{Y}\bolds{\bolds{\beta}}) +
\mathbf{D}_{(-i)} \bolds{\beta}\|_{\ell_\infty} \le\mu\|\bolds
{\beta }\|_{\ell
_1} +
\lambda$.\vadjust{\goodbreak}

To get a sense about the validity of this proposal, we test it on our
running example by varying $\lambda\in[\lambda_o,8\lambda_o]$ around
the heuristic $\lambda_o = \sqrt{2/n}   \sigma  \sqrt{1 +
\sigma^2}$. Figure~\ref{figImDantzig} shows that good results are
achieved around factors in the range $[4,6]$.

%
\begin{figure}

\includegraphics{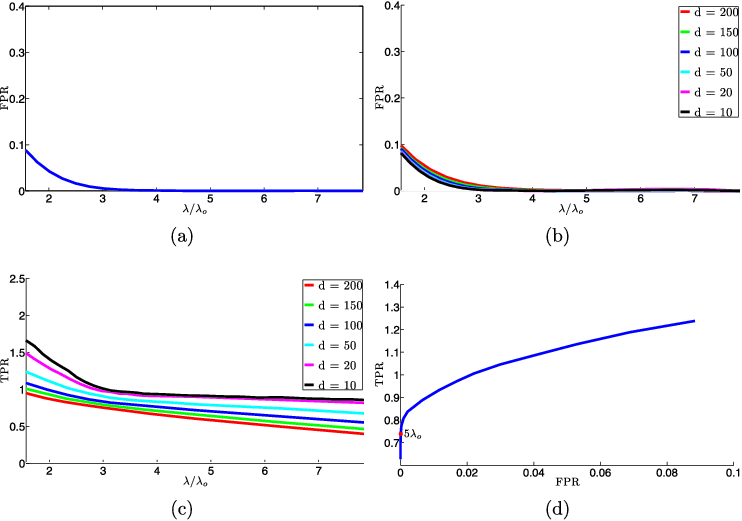}

\caption{Performance of the bias-corrected Dantzig selector
for values of $\lambda$ that are multiples of the heuristic
$\lambda_o = \sqrt{2/n}   \sigma  \sqrt{1 +
\sigma^2}$. \textup{(a)}~False positive rate (FPR). \textup{(b)}~FPR for
different subspace dimensions. \textup{(c)}~True positive rate
(TPR). \textup{(d)}~TPR vs. FPR.}\label{figImDantzig}
\end{figure}

In our synthetic simulations, both the two-step procedure and the
corrected Dantzig selector seem to be working well in the sense that
they yield many true discoveries while making very few false
discoveries, if any. Comparing Figure~\ref{figImDantzig}(b) and~(c) with those from Section~\ref{secmethod} show that the
corrected Dantzig selector has more true discoveries for subspaces of
small dimensions (they are essentially the same for subspaces of large
dimensions); that is, the two-step procedure is more conservative when
it comes to subspaces of smaller dimensions. As explained earlier, this
is due to our conservative choice of $\lambda$ resulting in a TPR
about half of what is obtained in a noiseless setting. Having said
this, it is important to keep in mind that in these simulations the
planes are drawn at random and as a result, they are sort of far from
each other. This is why a less conservative procedure can still
achieve a low FPR. When subspaces of smaller dimensions are closer to
each other or when the statistical model does not hold exactly as in
real data scenarios, a conservative procedure may be more
effective. In fact, experiments on real data in Section~\ref
{secnumerical} confirm this and show that for the corrected
Dantzig selector, one needs to choose values much larger than
$\lambda_o$ to yield good results.

\subsection{\texorpdfstring{Variance calculation.}{Variance calculation}}
\label{secvariance}

By definition,
\begin{eqnarray*}
\xi_j &=& \bigl\langle\mathbf{x}_j, \mathbf{z}_i-
\mathbf{Z}\bolds{\beta }^I\bigr\rangle+ \langle\mathbf{z}_j,
\mathbf{z}_i\rangle- \bigl(\mathbf{z}_j^T
\mathbf{z}_j- \sigma^2\bigr) \beta^I_j
- \sum_{k\dvtx  k \neq i, j} \mathbf{z}_j^T
\mathbf{z}_k \beta^I_k
\\
&:=& I_1 +I_2 + I_3 + I_4.
\end{eqnarray*}
A simple calculation shows that for $\ell_1 \neq\ell_2$,
$\operatorname{Cov}(I_{\ell_1},I_{\ell_2}) = 0$ so that
\[
\operatorname{\mathbb{E}}\xi_j^2 = \sum
_{\ell= 1}^4 \operatorname {Var}(I_\ell).
\]
We compute
\begin{eqnarray*}
\operatorname{Var}(I_1) & =& \frac{\sigma^2}{n}\bigl(1 + \bigl\|\bolds{\beta }^I\bigr\|_{\ell
_2}^2 \bigr),\qquad
\operatorname{Var}(I_3)  = \frac{\sigma^4}{n} 2\bigl(\beta_j^I \bigr)^2,
\\
\operatorname{Var}(I_2) & =& \frac{\sigma^4}{n},
\qquad
\operatorname{Var}(I_4)  = \frac{\sigma^4}{n} \bigl[\bigl\|\bolds{\beta
}^I\bigr\| _{\ell_2}^2 - \bigl(\beta_j^I
\bigr) ^2\bigr]
\end{eqnarray*}
and (\ref{eqvariance}) follows.

\section{\texorpdfstring{Comparisons with other works.}{Comparisons with other works}}
\label{comp}
We now briefly comment on other approaches to subspace
clustering. Since this paper is theoretical in nature, we shall focus
on comparing theoretical properties and refer to \cite{SSCalg,vidaltutorial} for a detailed comparison about empirical
performance. Three themes will help in organizing our discussion.
\begin{itemize}
\item\textit{Tractability}. Is the proposed method or algorithm
computationally tractable?
\item\textit{Robustness}. Is the algorithm provably robust to noise and
other imperfections?
\item\textit{Efficiency}. Is the algorithm correctly operating near the
limits we have identified above? In our model, how many points do we
need per subspace? How large can the affinity between subspaces be?
\end{itemize}

One can broadly classify existing subspace clustering techniques into
four categories, namely, algebraic, iterative, statistical and spectral
clustering-based
\mbox{methods}.

Methods inspired from algebraic geometry have been introduced for
clustering purposes. In this area, a mathematically intriguing
approach is the \emph{generalized principal component analysis} (GPCA)
presented in \cite{GPCA}. Unfortunately, this algorithm is not
tractable in the dimension of the subspaces, meaning that a
polynomial-time algorithm does not exist. Another feature is that
GPCA is not robust to noise although some heuristics have been
developed to address this issue; see, for example, \cite{noisyGPCA}. As
far as
the dependence upon key parameters is concerned, GPCA is essentially
optimal. An interesting approach to make GPCA robust is based on
semidefinite programming \cite{gpcaSDP}. However, this novel
formulation is still intractable in the dimension of the subspaces and
it is not clear how the performance of the algorithm depends upon the
parameters of interest.


A representative example of an iterative method---the term is taken
from the tutorial \cite{vidaltutorial}---is the $K$-subspace algorithm
\cite{ksubspaces}, a procedure which can be viewed as a generalization
of $K$-means. Here, the subspace clustering problem is formulated as a
nonconvex optimization problem over the choice of bases for each
subspace as well as a set of variables indicating the correct
segmentation. A cost function is then iteratively optimized over the
basis and the segmentation variables. Each iteration is
computationally tractable. However, due to the nonconvex nature of
the problem, the convergence of the sequence of iterates is only
guaranteed to a local minimum. As a consequence, the dependence upon
the key parameters is not well understood. Furthermore, the algorithm
can be sensitive to noise and outliers. Other examples of iterative
methods may be found in \cite{Bradley,agarwal,luvidal,mediankflat}.

Statistical methods typically model the subspace clustering problem as
a mixture of degenerate Gaussian observations. Two such approaches are
\emph{mixtures of probabilistic PCA} (MPPCA) \cite{MPPCA} and \emph{agglomerative lossy compression} (ALC) \cite{ALC}. MPPCA seeks to
compute a maximum-likelihood estimate of the parameters of the mixture
model by using an expected--maximization (EM) style algorithm. ALC
searches for a segmentation of the data by minimizing the code length
necessary (with a code based on Gaussian mixtures) to fit the points
up to a given distortion. Once more, due to the nonconvex nature of
these formulations, the dependence upon the key parameters and the
noise level is not understood.

Many other methods apply spectral clustering to a specially
constructed graph \cite{Boult,Yan,zhang2012hybrid,goh2007segmenting,SCC,SCCFOCM,arias,aldroubi2012nearness}. They
share the same difficulties as stated above and \cite{vidaltutorial}
discusses advantages and drawbacks. An approach of this kind is termed
\emph{Sparse Curvature Clustering} (SCC) \cite{SCC,SCCFOCM}; please
also see \cite{arias,ariasclust}. This approach is not tractable in
the dimension of the subspaces as it requires building a tensor with
$N^{(d+2)}$ entries and involves computations with this tensor. Some
theoretical guarantees for this algorithm are given in \cite{SCCFOCM}
although its limits of performance and robustness to noise are not
fully understood. An approach similar to SSC is called \emph{low-rank
representation} (LRR) \cite{LRR}. The LRR algorithm is tractable but
its robustness to noise and its dependence upon key parameters is not
understood. The work in \cite{lerman} formulates the robust subspace
clustering problem as a nonconvex geometric minimization problem over
the Grassmanian. Because of the nonconvexity, this formulation may
not be tractable. On the positive side, this algorithm is provably
robust and can accommodate noise levels up to
$\mathcal{O}(1/(Ld^{3/2}))$. However, the density $\rho$ required for
favorable properties to hold is an unknown function of the dimensions
of the subspaces (e.g., $\rho$ could depend on $d$ in a super
polynomial fashion). Also, the bound on the noise level seems to
decrease as the dimension $d$ and number of subspaces $L$
increases. In contrast, our theory requires $\rho\ge\rho^\star$ where
$\rho^\star$ is a fixed numerical constant. While this manuscript was
under preparation, we learned of \cite{highrankMC} which establishes
robustness to sparse outliers but with a dependence on the key
parameters that is super-polynomial in the dimension of the subspaces
demanding $\rho\ge C_0   d^{\log n}$. (Numerical simulations in
\cite{highrankMC} seem to indicate that $\rho$ cannot be a constant.)

We note that the papers \cite{MUS,ImprovedMUS,loh2012high} also address
regression under corrupted covariates. However, there are three key
differences between these studies and our work. First, our results
show that LASSO without any change is robust to corrupted covariates
whereas these works require modifications to either LASSO or the
Dantzig selector. Second, the modeling assumptions for the uncorrupted
covariates are significantly different. These papers assume that
$\mathbf{X}$
has i.i.d. rows and obeys the \emph{restricted eigenvalue condition}
(REC) whereas we have columns sampled from a mixture model so that the
design matrices do not have much in common. Last, for clustering and
classification purposes, we need to verify that the support of the
solution is correct whereas these works establish closeness to an
oracle solution in an $\ell_2$ sense. In short, our work is far closer
to multiple hypothesis testing.

Finally, in the data mining literature subspace clustering is
sometimes used to describe a different---although related---problem;
see \cite{muller2009evaluating,gunnemann2011flexible,keller2012hics}.

\section{\texorpdfstring{Numerical experiments.}{Numerical experiments}}
\label{secnumerical}

In this section, we perform numerical experiments corroborating our
main results and suggesting their applications to temporal
segmentation of motion capture data. In this application, we are given
sensor measurements at multiple joints of the human body captured at
different time instants. The goal is to segment the sensory data so
that each cluster corresponds to the same activity. Here, each data
point corresponds to a vector whose elements are the sensor
measurements of different joints at a fixed time instant.

We use the Carnegie Mellon Motion Capture dataset (available at \url{http://mocap.cs.cmu.edu}), which
contains 149 subjects performing several activities (data are provided
in \cite{MCdata}). The motion capture system uses 42 markers per
subject. We consider the data from subject $86$ in the dataset,
consisting of $15$ different trials, where each trial comprises
multiple activities. We use trials $2$ and $5$, which feature more
activities ($8$ activities for trial $2$ and $7$ activities for trial
$5$) and are, therefore, harder examples relative to the other
trials. Figure~\ref{figactivities} shows a few snapshots of each
activity (walking, squatting, punching, standing, running, jumping,
arms-up and drinking) from trial $2$. The right plot in Figure~\ref{figactivities} shows the singular values of three of the
activities in this trial. Notice that all the curves have a
low-dimensional knee, showing that the data from each activity lie in
a low-dimensional subspace of the ambient space ($n=42$ for all the
motion capture data).

%
%
\begin{figure}

\includegraphics{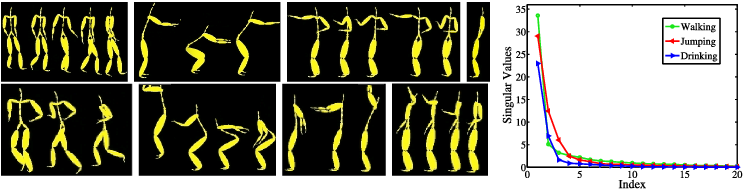}

\caption{Left: eight activities performed by subject $86$ in the CMU
motion capture dataset: walking, squatting, punching, standing,
running, jumping, arms-up and drinking. Right: singular values of the
data from three activities (walking, jumping, drinking) show that the
data from each activity lie approximately in a low-dimensional subspace.}\label{figactivities}
\end{figure}

We compare three different algorithms: a baseline algorithm, the
two-step procedure and the bias-corrected Dantzig selector. We
evaluate these algorithms based on the \emph{clustering error}. That
is, we assume knowledge of the number of subspaces and apply spectral
clustering to the similarity matrix built by the algorithm. After the
spectral clustering step, the clustering error is simply the ratio of
misclassified points to the total number of points. We report our
results on half of the examples---downsampling the video by a factor
$2$ keeping every other frame---as to make the problem more
challenging. (As a side note, it is always desirable to have methods
that work well on a smaller number of examples as one can use
split-sample strategies for tuning purposes.)\footnote{We have adopted
this subsampling strategy to make our experiments reproducible. For
tuning purposes, a random strategy may be preferable.}

As a baseline for comparison, we apply spectral clustering to a
standard similarity graph built by connecting each data point to its
$K$-nearest neighbors. For pairs of data points, $\mathbf{y}_i$ and $\mathbf{y}_j$,
that are connected in the $K$-nearest neighbor graph, we define the
similarities between them by $W_{ij} = \exp(-\| \mathbf{y}_i - \mathbf{y}_j
\|_2^2 /
t)$, where $t > 0$ is a tuning parameter (a.k.a. temperature). For
pairs of data points, $\mathbf{y}_i$ and $\mathbf{y}_j$, that are not
connected in the
$K$-nearest neighbor graph, we set $W_{ij} = 0$. Thus, pairs of
neighboring data points that have small Euclidean distances from each
other are considered to be more similar, since they have high
similarity $W_{ij}$. We then apply spectral clustering to the
similarity graph and measure the clustering error. For each value of
$K$, we record the minimum clustering error over different choices of
the temperature parameter $t >0$ as shown in Figure~\ref{figknn2}(a)~and~(b). The minimum clustering error for trials $2$ and
$5$ are $17.06\%$ and $12.47\%$.

%
\begin{figure}

\includegraphics{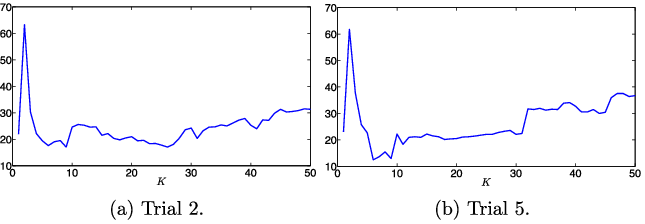}

\caption{Minimum clustering error ($\%$) for each $K$ in the baseline
algorithm.}\label{figknn2}
\end{figure}

For solving the LASSO problems in the two-step procedure, we developed
a computational routine made publicly available \cite{MSweb} based on
TFOCS \cite{TFOCS} solving the optimization problems in parallel. For
the corrected Dantzig selector, we use a homotopy solver in the spirit
of \cite{homotopy}.

For both the two-step procedure and the bias-corrected Dantzig
selector, we normalize the data points as a preprocessing step. We
work with a noise $\sigma$ in the interval $[0.001,0.045]$, and use
$f(t)=\alpha/t$ with values of $\alpha$ around $1/4$ (this is
equivalent to varying $\lambda$ around $1/\lambda_o =
4\llVert \bolds{\beta}^\star\rrVert _{\ell_1}$) in the two-step
procedure. For the
bias-corrected Dantzig selector, we vary $\lambda$ around $\lambda_o =
\sqrt{2/n}   \sigma  \sqrt{1 +\sigma^2}$. After building the
similarity graph from the sparse regression output, we apply spectral
clustering as explained earlier. Figures~\ref{figTS}(a)~and~(b), \ref{figImDan}(a) and~(b) show the
clustering error (on trial $5$) and the red point indicates the
location where the minimum clustering error is reached. Figure~\ref{figTS}(a) and~(b) shows that for the two-step procedure
the value of the clustering error is not overly sensitive to the
choice of $\sigma$---especially around $\lambda=\lambda_o$. Notice
that the clustering error for the robust versions of SSC are
significantly lower than the baseline algorithm for a wide range of
parameter values. The reason the baseline algorithm performs poorly in
this case is that there are many points that are in small Euclidean
distances from each other, but belong to different subspaces.


%
\begin{figure}

\includegraphics[scale=0.97]{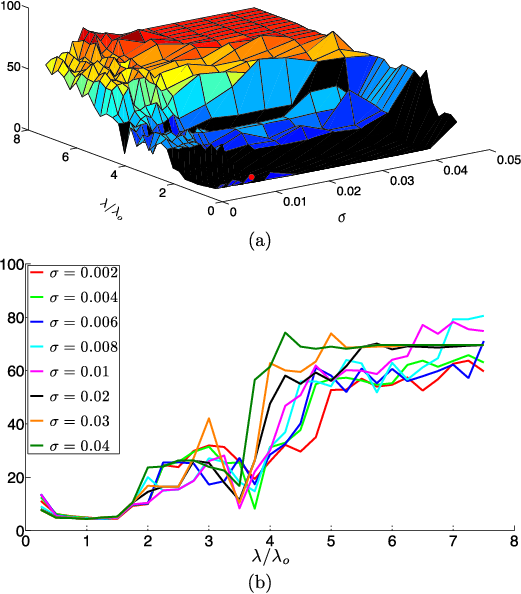}

\caption{Clustering error (\%) for different values of $\lambda$
and $\sigma$ on trial $5$ using the two-step procedure \textup{(a)}~3D plot
(minimum clustering error appears
in red). \textup{(b)}~2D cross sections.}\label{figTS}
\end{figure}

Finally a summary of the clustering errors of these algorithms on the
two trials are reported in Table~\ref{table1}. Robust versions of SSC outperform
the baseline algorithm. This shows that the multiple subspace model is
better for clustering purposes. The two-step procedure seems to work
slightly better than the corrected Dantzig selector for these two
examples. Table~\ref{table2} reports the optimal parameters that achieve the
minimum clustering error for each algorithm. The table indicates that
on real data, choosing $\lambda$ close to $\lambda_o$ also works very
well. Also, one can see that in comparison with the synthetic
simulations of Section~\ref{dantzigsection}, a more conservative
choice of the regularization parameter $\lambda$ is needed for the
corrected Dantzig selector as $\lambda$ needs to be chosen much higher
than $\lambda_o$ to achieve the best results. This may be attributed
to the fact that the subspaces in this example are very close to each
other and are not drawn at random as was the case with our synthetic
data. To get a sense of the affinity values, we fit a subspace of
dimension $d_\ell$ to the $N_\ell$ data points from the $\ell$th
group, where $d_\ell$ is chosen as the smallest nonnegative integer
such that the partial sum of the $d_\ell$ top singular values is at
least 90\% of the total sum. Figure~\ref{figboxplot} shows that the
affinities are higher than $0.75$ for both trials.

%

\section{\texorpdfstring{Discussion and open problems.}{Discussion and open problems}}\label{discussion}

In this paper, we have developed a tractable algorithm that can
provably cluster data points in a fairly challenging regime in which
subspaces can overlap along many dimensions and in which the number of
points per subspace is rather limited.
%
\begin{figure}[t]

\includegraphics[scale=0.97]{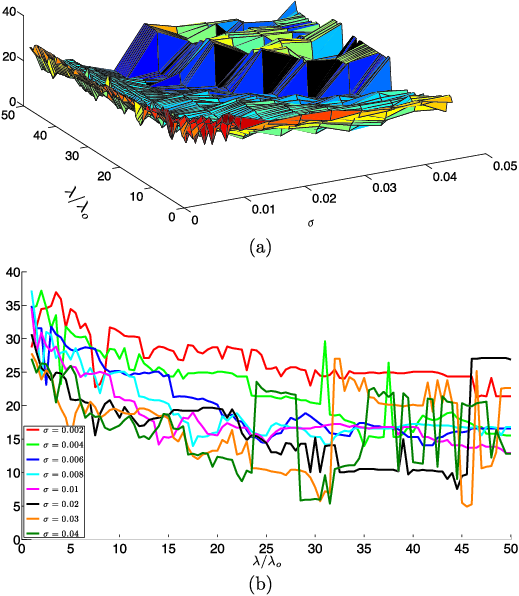}

\caption{Clustering error (\%) for different values of $\lambda$
and $\sigma$ on trial $5$ using the corrected Dantzig selector. \textup{(a)}~3D
plot (minimum clustering error appears
in red). \textup{(b)}~2D cross sections.}\label{figImDan}\vspace*{-3pt}
\end{figure}
%
\begin{table}[b]
\caption{Minimum clustering error}\label{table1}
\tabcolsep=12pt
\begin{tabular*}{\tablewidth}{@{\extracolsep{\fill}}@{}lcc cc @{}}
\hline
& \textbf{Baseline algorithm} & \textbf{Two-step procedure} & \textbf{Corrected Dantzig selector}\\
\hline
Trial 2 & 17.06\%& \textbf{3.54\%} & 9.53\% \\
Trial 5 & 12.47\%& \textbf{4.35\%} & 4.92\% \\
\hline
\end{tabular*}
\end{table}
Our results about the
performance of the robust SSC algorithm are expressed in terms of
interpretable parameters. This is not a trivial achievement: one of
the challenges of the theory for subspace clustering is precisely that
performance depends on many different aspects of the problem such as
the dimension of the ambient space, the number of subspaces, their
dimensions, their relative orientations, the distribution of points
around each subspace, the noise level and so on.\vadjust{\goodbreak}
Nevertheless, these results only offer a starting point as our work
leaves open lots of questions, and at the same time, suggests topics
for future research. Before presenting the proofs, we would like to
close by listing a few questions colleagues may find of interest.
\begin{itemize}
\item We have shown that while having the affinities and sampling
densities near what is information theoretically possible, robust
versions of SSC that can accommodate noise levels $\sigma$ of order
one exist. It would be interesting to establish fundamental limits
relating the key parameters to the maximum allowable noise
level. What is the maximum allowable noise level for any algorithm
regardless of tractability?\vadjust{\goodbreak}

\item It would be interesting to extend the results of this paper to a
deterministic model where both the orientation of the subspaces and
the noiseless samples are nonrandom. We leave this to a future
publication.

\begin{table}[t]
\caption{Optimal parameters}\label{table2}
\tabcolsep=12pt
\begin{tabular*}{\tablewidth}{@{\extracolsep{\fill}}lcc c c@{}}
\hline
& \textbf{Baseline algorithm} & \textbf{Two-step procedure} & \textbf{Corrected Dantzig selector}\\
\hline
Trial 2 & $K=9$, $t=0.0769$& $\sigma=0.03$, $\lambda=1.25\lambda_o$ & $\sigma=0.004$, $\lambda=41.5\lambda_o$\\
Trial 5 & $K=6$, $t=0.0455$& $\sigma=0.01$, $\lambda=\lambda_o$ & $\sigma=0.03$, $\lambda=45.5\lambda_o$ \\
\hline
\end{tabular*}
\end{table}
%
\begin{figure}[b]

\includegraphics{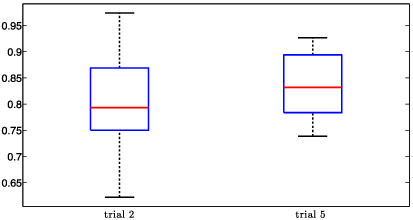}

\caption{Box plot of the affinities between subspaces for trials $2$
and $5$.}\label{figboxplot}
\end{figure}


\item Our work in this paper concerns the construction of the
similarity matrix and the correctness of sparse regression
techniques. The full algorithm then applies clustering techniques to
clean up errors introduced in the first step. It would be
interesting to develop theoretical guarantees for this step as
well. A potential approach is the interesting formulation developed
in~\cite{nina}.

\item We proposed a two-step procedure for robust subspace
clustering. The first step is used to estimate the required
regularization parameter for a LASSO problem. This is reminiscent of
estimating noise in sparse regularization and covariance estimation.
It would be interesting to design a joint optimization scheme to
simultaneous optimize the regularization parameter and the
regression coefficients. In recent years, there has been much
progress on this issue in the sparse regression literature; see
\cite{belloni2011square,giraud2012high,sun2012scaled,st2010l1,dalalyan2012fused} and references therein. It is an open research
direction to see whether any of these approaches can be applied to
automatically learn the regularization parameter when both the
response vector and covariates are corrupted and, in particular, for
the purpose of robust subspace clustering.\looseness=-1

\item A natural direction is the development of clustering techniques
that can provably operate with missing and/or sparsely corrupted
entries (the work \cite{ourSSC} only deals with grossly corrupted
columns). The work in \cite{highrankMC} provides one possible
approach but requires a very high sampling density as we already
mentioned. The paper \cite{SSCalg} develops another heuristic
approach without any theoretical justification.




\item{Our formulation uses a data-driven modeling approach by
regressing each data point against all others. As noted by Bittorf
et al. \cite{bittorf2012factoring}, this type of approach appears
in a number of other factorization problems. In particular,
\cite{arora2012computing} and recent variations
\cite{arora2012computing,bittorf2012factoring} use a convex
formulation very similar to SSC for the purpose of nonnegative
matrix factorizations. Exploring the connection between these
factorization problems is an interesting research direction.}

\item One of the advantages of the suggested scheme is that it is
highly parallelizable. When the algorithm is run sequentially, it
would be interesting to see whether one can reuse computations to
solve all the $\ell_1$-minimization problems more effectively.\vspace*{-3pt}
\end{itemize}

\section*{\texorpdfstring{Acknowledgements.}{Acknowledgements}}
We thank Ren\'e Vidal for helpful discussions as well as
Ery Arias-Castro, Rina Foygel and Lester Mackey for a careful reading
of the manuscript and insightful comments. We also thank the Associate
Editor and reviewers for constructive comments.
Emmanuel~J. Cand\`{e}s would like to thank Chiara Sabatti for invaluable feedback on an earlier version of
the paper. He also thanks the organizers of the 41st annual Meeting of
Dutch Statisticians and Probabilists held in November 2012 where these
results were presented. A brief \mbox{summary} of this work was submitted in
August 2012 and presented at the NIPS workshop on Deep Learning in
December 2012.\looseness=-1\vspace*{-3pt}

\begin{supplement}
\stitle{Supplement: Proofs}
\slink[doi]{10.1214/13-AOS1199SUPP} 
\sdatatype{.pdf}
\sfilename{AOS1199\_supp.pdf}
\sdescription{We prove all of the results of this paper.\vspace*{-3pt}}
\end{supplement}



%

\printaddresses


\begin{thebibliography}{61}
\bibitem{agarwal}
%
\begin{bincollection}[auto:STB|2014/02/12|14:17:21]
\bauthor{\bsnm{Agarwal},~\bfnm{P.~K.}\binits{P.~K.}} \AND
\bauthor{\bsnm{Mustafa},~\bfnm{N.~H.}\binits{N.~H.}}
(\byear{2004}).
\btitle{$k$-means projective clustering}.
In \bbooktitle{Proceedings of the Twenty-third ACM SIGMOD-SIGACT-SIGART
Symposium on Principles of Database Systems}
\bpages{155--165}.
\end{bincollection}
%
\bptok{imsref}%
\endbibitem

\bibitem{aldroubi2012nearness}
%
\begin{barticle}[auto:STB|2014/02/12|14:17:21]
\bauthor{\bsnm{Aldroubi},~\bfnm{A.}\binits{A.}} \AND
\bauthor{\bsnm{Sekmen},~\bfnm{A.}\binits{A.}}
(\byear{2012}).
\btitle{Nearness to local subspace algorithm for subspace and motion
segmentation}.
\bjournal{Signal Process. Lett., IEEE}
\bvolume{19}
\bpages{704--707}.
\end{barticle}
%
\bptok{imsref}%
\endbibitem

\bibitem{ariasclust}
%
\begin{barticle}[mr]
\bauthor{\bsnm{Arias-Castro},~\bfnm{Ery}\binits{E.}}
(\byear{2011}).
\btitle{Clustering based on pairwise distances when the data is of
mixed dimensions}.
\bjournal{IEEE Trans. Inform. Theory}
\bvolume{57}
\bpages{1692--1706}.
\bid{doi={10.1109/TIT.2011.2104630}, issn={0018-9448}, mr={2815843}}
\end{barticle}
%
\bptok{imsref}%
\endbibitem

\bibitem{arias}
%
\begin{barticle}[mr]
\bauthor{\bsnm{Arias-Castro},~\bfnm{Ery}\binits{E.}},
\bauthor{\bsnm{Chen},~\bfnm{Guangliang}\binits{G.}} \AND
\bauthor{\bsnm{Lerman},~\bfnm{Gilad}\binits{G.}}
(\byear{2011}).
\btitle{Spectral clustering based on local linear approximations}.
\bjournal{Electron. J. Stat.}
\bvolume{5}
\bpages{1537--1587}.
\bid{doi={10.1214/11-EJS651}, issn={1935-7524}, mr={2861697}}
\end{barticle}
%
\bptok{imsref}%
\endbibitem

\bibitem{arora2012computing}
%
\begin{bincollection}[mr]
\bauthor{\bsnm{Arora},~\bfnm{Sanjeev}\binits{S.}},
\bauthor{\bsnm{Ge},~\bfnm{Rong}\binits{R.}},
\bauthor{\bsnm{Kannan},~\bfnm{Ravi}\binits{R.}} \AND
\bauthor{\bsnm{Moitra},~\bfnm{Ankur}\binits{A.}}
(\byear{2012}).
\btitle{Computing a nonnegative matrix factorization--provably}.
In \bbooktitle{S{TOC}'12---{P}roceedings of the 2012 {ACM} {S}ymposium
on {T}heory of {C}omputing}
\bpages{145--161}.
\bpublisher{ACM},
\blocation{New York}.
\bid{doi={10.1145/2213977.2213994}, mr={2961503}}
\end{bincollection}
%
\bptok{imsref}%
\endbibitem

\bibitem{sysID}
%
\begin{barticle}[mr]
\bauthor{\bsnm{Bako},~\bfnm{Laurent}\binits{L.}}
(\byear{2011}).
\btitle{Identification of switched linear systems via sparse optimization}.
\bjournal{Automatica J. IFAC}
\bvolume{47}
\bpages{668--677}.
\bid{doi={10.1016/j.automatica.2011.01.036}, issn={0005-1098}, mr={2878328}}
\end{barticle}
%
\bptok{imsref}%
\endbibitem

\bibitem{nina}
%
\begin{binproceedings}[mr]
\bauthor{\bsnm{Balcan},~\bfnm{Maria-Florina}\binits{M.-F.}},
\bauthor{\bsnm{Blum},~\bfnm{Avrim}\binits{A.}} \AND
\bauthor{\bsnm{Gupta},~\bfnm{Anupam}\binits{A.}}
(\byear{2009}).
\btitle{Approximate clustering without the approximation}.
In \bbooktitle{Proceedings of the {T}wentieth {A}nnual {ACM}--{SIAM}
{S}ymposium on {D}iscrete {A}lgorithms}
\bpages{1068--1077}.
\bpublisher{SIAM},
\blocation{Philadelphia, PA}.
\bid{mr={2807549}}
\end{binproceedings}
%
\bptok{imsref}%
\endbibitem\vadjust{\goodbreak}

\bibitem{DMP}
%
\begin{barticle}[mr]
\bauthor{\bsnm{Bayati},~\bfnm{Mohsen}\binits{M.}} \AND
\bauthor{\bsnm{Montanari},~\bfnm{Andrea}\binits{A.}}
(\byear{2011}).
\btitle{The dynamics of message passing on dense graphs, with
applications to compressed sensing}.
\bjournal{IEEE Trans. Inform. Theory}
\bvolume{57}
\bpages{764--785}.
\bid{doi={10.1109/TIT.2010.2094817}, issn={0018-9448}, mr={2810285}}
\end{barticle}
%
\bptok{imsref}%
\endbibitem

\bibitem{BMLasso}
%
\begin{barticle}[mr]
\bauthor{\bsnm{Bayati},~\bfnm{Mohsen}\binits{M.}} \AND
\bauthor{\bsnm{Montanari},~\bfnm{Andrea}\binits{A.}}
(\byear{2012}).
\btitle{The {LASSO} risk for {G}aussian matrices}.
\bjournal{IEEE Trans. Inform. Theory}
\bvolume{58}
\bpages{1997--2017}.
\bid{doi={10.1109/TIT.2011.2174612}, issn={0018-9448}, mr={2951312}}
\end{barticle}
%
\bptok{imsref}%
\endbibitem

\bibitem{TFOCS}
%
\begin{barticle}[mr]
\bauthor{\bsnm{Becker},~\bfnm{Stephen~R.}\binits{S.~R.}},
\bauthor{\bsnm{Cand{\`e}s},~\bfnm{Emmanuel~J.}\binits{E.~J.}} \AND
\bauthor{\bsnm{Grant},~\bfnm{Michael~C.}\binits{M.~C.}}
(\byear{2011}).
\btitle{Templates for convex cone problems with applications to sparse
signal recovery}.
\bjournal{Math. Program. Comput.}
\bvolume{3}
\bpages{165--218}.
\bid{doi={10.1007/s12532-011-0029-5}, issn={1867-2949}, mr={2833262}}
\end{barticle}
%
\bptok{imsref}%
\endbibitem

\bibitem{belloni2011square}
%
\begin{barticle}[mr]
\bauthor{\bsnm{Belloni},~\bfnm{A.}\binits{A.}},
\bauthor{\bsnm{Chernozhukov},~\bfnm{V.}\binits{V.}} \AND
\bauthor{\bsnm{Wang},~\bfnm{L.}\binits{L.}}
(\byear{2011}).
\btitle{Square-root lasso: Pivotal recovery of sparse signals via conic
programming}.
\bjournal{Biometrika}
\bvolume{98}
\bpages{791--806}.
\bid{doi={10.1093/biomet/asr043}, issn={0006-3444}, mr={2860324}}
\end{barticle}
%
\bptok{imsref}%
\endbibitem

\bibitem{bittorf2012factoring}
%
\begin{bmisc}[auto:STB|2014/02/12|14:17:21]
\bauthor{\bsnm{Bittorf},~\bfnm{V.}\binits{V.}},
\bauthor{\bsnm{Recht},~\bfnm{B.}\binits{B.}},
\bauthor{\bsnm{Re},~\bfnm{C.}\binits{C.}} \AND
\bauthor{\bsnm{Tropp},~\bfnm{J.~A.}\binits{J.~A.}}
(\byear{2012}).
\bhowpublished{Factoring nonnegative matrices with linear programs.
In \textit{Proceedings of Natural Information Processing Systems Foundation NIPS}.}
\end{bmisc}
%
\bptok{imsref}%
\endbibitem

\bibitem{Boult}
%
\begin{bincollection}[auto:STB|2014/02/12|14:17:21]
\bauthor{\bsnm{Boult},~\bfnm{T.~E.}\binits{T.~E.}} \AND
\bauthor{\bsnm{Gottesfeld Brown},~\bfnm{L.}\binits{L.}}
(\byear{1991}).
\btitle{Factorization-based segmentation of motions}.
In \bbooktitle{Proceedings of the IEEE Workshop on Visual Motion}
\bpages{179--186}.
\end{bincollection}
%
\bptok{imsref}%
\endbibitem

\bibitem{Bradley}
%
\begin{barticle}[mr]
\bauthor{\bsnm{Bradley},~\bfnm{P.~S.}\binits{P.~S.}} \AND
\bauthor{\bsnm{Mangasarian},~\bfnm{O.~L.}\binits{O.~L.}}
(\byear{2000}).
\btitle{{$k$}-plane clustering}.
\bjournal{J. Global Optim.}
\bvolume{16}
\bpages{23--32}.
\bid{doi={10.1023/A:1008324625522}, issn={0925-5001}, mr={1770524}}
\end{barticle}
%
\bptok{imsref}%
\endbibitem

\bibitem{dantzig}
%
\begin{barticle}[mr]
\bauthor{\bsnm{Cand\`es},~\bfnm{Emmanuel}\binits{E.}} \AND
\bauthor{\bsnm{Tao},~\bfnm{Terence}\binits{T.}}
(\byear{2007}).
\btitle{The {D}antzig selector: Statistical estimation when {$p$} is
much larger than {$n$}}.
\bjournal{Ann. Statist.}
\bvolume{35}
\bpages{2313--2351}.
\bid{doi={10.1214/009053606000001523}, issn={0090-5364}, mr={2382644}}
\end{barticle}
%
\bptok{imsref}%
\endbibitem

\bibitem{SCCFOCM}
%
\begin{barticle}[mr]
\bauthor{\bsnm{Chen},~\bfnm{Guangliang}\binits{G.}} \AND
\bauthor{\bsnm{Lerman},~\bfnm{Gilad}\binits{G.}}
(\byear{2009}).
\btitle{Foundations of a multi-way spectral clustering framework for
hybrid linear modeling}.
\bjournal{Found. Comput. Math.}
\bvolume{9}
\bpages{517--558}.
\bid{doi={10.1007/s10208-009-9043-7}, issn={1615-3375}, mr={2534403}}
\end{barticle}
%
\bptok{imsref}%
\endbibitem

\bibitem{SCC}
%
\begin{barticle}[auto:STB|2014/02/12|14:17:21]
\bauthor{\bsnm{Chen},~\bfnm{G.}\binits{G.}} \AND
\bauthor{\bsnm{Lerman},~\bfnm{G.}\binits{G.}}
(\byear{2009}).
\btitle{Spectral curvature clustering (SCC)}.
\bjournal{Int. J. Comput. Vis.}
\bvolume{81}
\bpages{317--330}.
\end{barticle}
%
\bptok{imsref}%
\endbibitem

\bibitem{hyperspectral}
%
\begin{barticle}[auto:STB|2014/02/12|14:17:21]
\bauthor{\bsnm{Chen},~\bfnm{Y.}\binits{Y.}},
\bauthor{\bsnm{Nasrabadi},~\bfnm{N.~M.}\binits{N.~M.}} \AND
\bauthor{\bsnm{Tran},~\bfnm{T.~D.}\binits{T.~D.}}
(\byear{2011}).
\btitle{Hyperspectral image classification using dictionary-based
sparse representation}.
\bjournal{IEEE Trans. Geosci. Remote Sens.}
\bvolume{99}
\bpages{1--13}.
\end{barticle}
%
\bptok{imsref}%
\endbibitem

\bibitem{dalalyan2012fused}
%
\begin{bincollection}[auto:STB|2014/02/12|14:17:21]
\bauthor{\bsnm{Dalalyan},~\bfnm{A.}\binits{A.}} \AND
\bauthor{\bsnm{Chen},~\bfnm{Y.}\binits{Y.}}
(\byear{2012}).
\btitle{Fused sparsity and robust estimation for linear models with
unknown variance}.
In \bbooktitle{Advances in Neural Information Processing Systems}
\bvolume{25}
\bpages{1268--1276}.
\end{bincollection}
%
\bptok{imsref}%
\endbibitem

\bibitem{ehsanSSC}
%
\begin{bincollection}[auto:STB|2014/02/12|14:17:21]
\bauthor{\bsnm{Elhamifar},~\bfnm{E.}\binits{E.}} \AND
\bauthor{\bsnm{Vidal},~\bfnm{R.}\binits{R.}}
(\byear{2009}).
\btitle{Sparse subspace clustering}.
In \bbooktitle{IEEE Conference on Computer Vision and Pattern
Recognition, CVPR}
\bpages{2790--2797}.
\end{bincollection}
%
\bptok{imsref}%
\endbibitem

\bibitem{elhamifar2010clustering}
%
\begin{bincollection}[auto:STB|2014/02/12|14:17:21]
\bauthor{\bsnm{Elhamifar},~\bfnm{E.}\binits{E.}} \AND
\bauthor{\bsnm{Vidal},~\bfnm{R.}\binits{R.}}
(\byear{2010}).
\btitle{Clustering disjoint subspaces via sparse representation}.
In \bbooktitle{IEEE International Conference on Acoustics Speech and
Signal Processing, ICASSP}
\bpages{1926--1929}.
\bpublisher{IEEE Press}, \blocation{New York}.
\end{bincollection}
%
\bptok{imsref}%
\endbibitem

\bibitem{SSCalg}
%
\begin{barticle}[auto]
\bauthor{\bsnm{Elhamifar},~\bfnm{Ehsan}\binits{E.}} \AND
\bauthor{\bsnm{Vidal},~\bfnm{Ren{\'e}}\binits{R.}}
(\byear{2013}).
\btitle{Sparse subspace clustering: Algorithms, theory, and applications}.
\bjournal{IEEE Trans. Pattern Anal. Mach. Intell.}
\bvolume{35}
\bpages{2765--2781}.
\end{barticle}
%
\bptok{imsref}%
\endbibitem


\bibitem{highrankMC}
%
\begin{bmisc}[auto:STB|2014/02/12|14:17:21]
\bauthor{\bsnm{Eriksson},~\bfnm{B.}\binits{B.}},
\bauthor{\bsnm{Balzano},~\bfnm{L.}\binits{L.}} \AND
\bauthor{\bsnm{Nowak},~\bfnm{R.}\binits{R.}}
(\byear{2011}).
\bhowpublished{High-rank matrix completion and subspace clustering with
missing data. Preprint. Available at \arxivurl{arXiv:1112.5629}.}
\end{bmisc}
%
\bptok{imsref}%
\endbibitem

\bibitem{giraud2012high}
%
\begin{barticle}[mr]
\bauthor{\bsnm{Giraud},~\bfnm{Christophe}\binits{C.}},
\bauthor{\bsnm{Huet},~\bfnm{Sylvie}\binits{S.}} \AND
\bauthor{\bsnm{Verzelen},~\bfnm{Nicolas}\binits{N.}}
(\byear{2012}).
\btitle{High-dimensional regression with unknown variance}.
\bjournal{Statist. Sci.}
\bvolume{27}
\bpages{500--518}.
\bid{doi={10.1214/12-STS398}, issn={0883-4237}, mr={3025131}}
\end{barticle}
%
\bptok{imsref}%
\endbibitem

\bibitem{goh2007segmenting}
%
\begin{bincollection}[auto:STB|2014/02/12|14:17:21]
\bauthor{\bsnm{Goh},~\bfnm{A.}\binits{A.}} \AND
\bauthor{\bsnm{Vidal},~\bfnm{R.}\binits{R.}}
(\byear{2007}).
\btitle{Segmenting motions of different types by unsupervised manifold
clustering}.
In \bbooktitle{IEEE International Conference on Computer Vision and
Pattern Recognition, CVPR}
\bpages{1--6}.
\bpublisher{IEEE Press}, \blocation{New York}.
\end{bincollection}
%
\bptok{imsref}%
\endbibitem

\bibitem{gunnemann2011flexible}
%
\begin{bincollection}[auto:STB|2014/02/12|14:17:21]
\bauthor{\bsnm{Gunnemann},~\bfnm{S.}\binits{S.}},
\bauthor{\bsnm{Muller},~\bfnm{E.}\binits{E.}},
\bauthor{\bsnm{Raubach},~\bfnm{S.}\binits{S.}} \AND
\bauthor{\bsnm{Seidl},~\bfnm{T.}\binits{T.}}
(\byear{2011}).
\btitle{Flexible fault tolerant subspace clustering for data with
missing values}.
In \bbooktitle{IEEE International Conference on Data Mining, ICDM}
\bpages{231--240}.
\end{bincollection}
%
\bptok{imsref}%
\endbibitem

\bibitem{kannan2009spectral}
%
\begin{barticle}[mr]
\bauthor{\bsnm{Kannan},~\bfnm{Ravindran}\binits{R.}} \AND
\bauthor{\bsnm{Vempala},~\bfnm{Santosh}\binits{S.}}
(\byear{2008}).
\btitle{Spectral algorithms}.
\bjournal{Found. Trends Theor. Comput. Sci.}
\bvolume{4}
\bpages{157--288 (2009)}.
\bid{doi={10.1561/0400000025}, issn={1551-305X}, mr={2558901}}
\bptnote{check year}%
\end{barticle}
%
\bptok{imsref}%
\endbibitem

\bibitem{keller2012hics}
%
\begin{bincollection}[auto:STB|2014/02/12|14:17:21]
\bauthor{\bsnm{Keller},~\bfnm{F.}\binits{F.}},
\bauthor{\bsnm{Muller},~\bfnm{E.}\binits{E.}} \AND
\bauthor{\bsnm{Bohm},~\bfnm{K.}\binits{K.}}
(\byear{2012}).
\btitle{HICS: High contrast subspaces for density-based outlier ranking}.
In \bbooktitle{IEEE International Conference on Data Engineering, ICDE}
\bpages{1037--1048}.
\end{bincollection}
%
\bptok{imsref}%
\endbibitem

\bibitem{music}
%
\begin{bincollection}[auto:STB|2014/02/12|14:17:21]
\bauthor{\bsnm{Kotropoulos},~\bfnm{Y.~P.~C.}\binits{Y.~P.~C.}} \AND
\bauthor{\bsnm{Arce},~\bfnm{G.~R.}\binits{G.~R.}}
(\byear{2011}).
\btitle{$\ell_1$-graph based music structure analysis}.
In \bbooktitle{International Society for Music Information Retrieval
Conference, ISMIR}.
\end{bincollection}
%
\bptok{imsref}%
\endbibitem

\bibitem{lerman}
%
\begin{barticle}[mr]
\bauthor{\bsnm{Lerman},~\bfnm{Gilad}\binits{G.}} \AND
\bauthor{\bsnm{Zhang},~\bfnm{Teng}\binits{T.}}
(\byear{2011}).
\btitle{Robust recovery of multiple subspaces by geometric {$l_ p$}
minimization}.
\bjournal{Ann. Statist.}
\bvolume{39}
\bpages{2686--2715}.
\bid{doi={10.1214/11-AOS914}, issn={0090-5364}, mr={2906883}}
\end{barticle}
%
\bptok{imsref}%
\endbibitem

\bibitem{LRR}
%
\begin{barticle}[auto:STB|2014/02/12|14:17:21]
\bauthor{\bsnm{Liu},~\bfnm{G.}\binits{G.}},
\bauthor{\bsnm{Lin},~\bfnm{Z.}\binits{Z.}},
\bauthor{\bsnm{Yan},~\bfnm{S.}\binits{S.}},
\bauthor{\bsnm{Sun},~\bfnm{J.}\binits{J.}},
\bauthor{\bsnm{Yu},~\bfnm{Y.}\binits{Y.}} \AND
\bauthor{\bsnm{Ma},~\bfnm{Y.}\binits{Y.}}
(\byear{2013}).
\btitle{Robust recovery of subspace structures by low-rank representation}.
\bjournal{IEEE Trans. Pattern Anal. Mach. Intell.}
\bvolume{35}
\bpages{171--184}.
\end{barticle}
%
\bptok{imsref}%
\endbibitem

\bibitem{loh2012high}
%
\begin{barticle}[mr]
\bauthor{\bsnm{Loh},~\bfnm{Po-Ling}\binits{P.-L.}} \AND
\bauthor{\bsnm{Wainwright},~\bfnm{Martin~J.}\binits{M.~J.}}
(\byear{2012}).
\btitle{High-dimensional regression with noisy and missing data:
Provable guarantees with nonconvexity}.
\bjournal{Ann. Statist.}
\bvolume{40}
\bpages{1637--1664}.
\bid{doi={10.1214/12-AOS1018}, issn={0090-5364}, mr={3015038}}
\end{barticle}
%
\bptok{imsref}%
\endbibitem

\bibitem{luvidal}
%
\begin{bincollection}[auto:STB|2014/02/12|14:17:21]
\bauthor{\bsnm{Lu},~\bfnm{L.}\binits{L.}} \AND
\bauthor{\bsnm{Vidal},~\bfnm{R.}\binits{R.}}
(\byear{2006}).
\btitle{Combined central and subspace clustering for computer vision
applications}.
In \bbooktitle{Proceedings of the 23rd International Conference on
Machine Learning}
\bpages{593--600}.
\bpublisher{ACM}, \blocation{New York}.
\end{bincollection}
%
\bptok{imsref}%
\endbibitem

\bibitem{ALC}
%
\begin{barticle}[auto:STB|2014/02/12|14:17:21]
\bauthor{\bsnm{Ma},~\bfnm{Y.}\binits{Y.}},
\bauthor{\bsnm{Derksen},~\bfnm{H.}\binits{H.}},
\bauthor{\bsnm{Hong},~\bfnm{W.}\binits{W.}} \AND
\bauthor{\bsnm{Wright},~\bfnm{J.}\binits{J.}}
(\byear{2007}).
\btitle{Segmentation of multivariate mixed data via lossy data coding
and compression}.
\bjournal{IEEE Trans. Pattern Anal. Mach. Intell.}
\bvolume{29}
\bpages{1546--1562}.
\end{barticle}
%
\bptok{imsref}%
\endbibitem

\bibitem{masysid}
%
\begin{bincollection}[auto:STB|2014/02/12|14:17:21]
\bauthor{\bsnm{Ma},~\bfnm{Y.}\binits{Y.}} \AND
\bauthor{\bsnm{Vidal},~\bfnm{R.}\binits{R.}}
(\byear{2005}).
\btitle{Identification of deterministic switched arx systems via identification  of algebraic varieties}.
In \bbooktitle{Hybrid Systems: Computation and Control}
\bpages{449--465}.
\end{bincollection}
%
\bptok{imsref}%
\endbibitem

\bibitem{noisyGPCA}
%
\begin{barticle}[mr]
\bauthor{\bsnm{Ma},~\bfnm{Yi}\binits{Y.}},
\bauthor{\bsnm{Yang},~\bfnm{Allen~Y.}\binits{A.~Y.}},
\bauthor{\bsnm{Derksen},~\bfnm{Harm}\binits{H.}} \AND
\bauthor{\bsnm{Fossum},~\bfnm{Robert}\binits{R.}}
(\byear{2008}).
\btitle{Estimation of subspace arrangements with applications in
modeling and segmenting mixed data}.
\bjournal{SIAM Rev.}
\bvolume{50}
\bpages{413--458}.
\bid{doi={10.1137/060655523}, issn={0036-1445}, mr={2429444}}
\end{barticle}
%
\bptok{imsref}%
\endbibitem

\bibitem{montana}
%
\begin{barticle}[mr]
\bauthor{\bsnm{McWilliams},~\bfnm{Brian}\binits{B.}} \AND
\bauthor{\bsnm{Montana},~\bfnm{Giovanni}\binits{G.}}
(\byear{2014}).
\btitle{Subspace clustering of high-dimensional data: A predictive approach}.
\bjournal{Data Min. Knowl. Discov.}
\bvolume{28}
\bpages{736--772}.
\bid{doi={10.1007/s10618-013-0317-y}, issn={1384-5810}, mr={3165525}}
\bptnote{check year}%
\end{barticle}
%
\bptok{imsref}%
\endbibitem

\bibitem{GLASSO}
%
\begin{barticle}[mr]
\bauthor{\bsnm{Meinshausen},~\bfnm{Nicolai}\binits{N.}} \AND
\bauthor{\bsnm{B{\"u}hlmann},~\bfnm{Peter}\binits{P.}}
(\byear{2006}).
\btitle{High-dimensional graphs and variable selection with the lasso}.
\bjournal{Ann. Statist.}
\bvolume{34}
\bpages{1436--1462}.
\bid{doi={10.1214/009053606000000281}, issn={0090-5364}, mr={2278363}}
\end{barticle}
%
\bptok{imsref}%
\endbibitem

\bibitem{muller2009evaluating}
%
\begin{barticle}[auto:STB|2014/02/12|14:17:21]
\bauthor{\bsnm{M{\"u}ller},~\bfnm{E.}\binits{E.}},
\bauthor{\bsnm{Gunnemann},~\bfnm{S.}\binits{S.}},
\bauthor{\bsnm{Assent},~\bfnm{I.}\binits{I.}} \AND
\bauthor{\bsnm{Seidl},~\bfnm{T.}\binits{T.}}
(\byear{2009}).
\btitle{Evaluating clustering in subspace projections of high
dimensional data}.
\bjournal{Proc. VLDB Endow.}
\bvolume{2}
\bpages{1270--1281}.
\end{barticle}
%
\bptok{imsref}%
\endbibitem

\bibitem{ng2002spectral}
%
\begin{barticle}[auto:STB|2014/02/12|14:17:21]
\bauthor{\bsnm{Ng},~\bfnm{A.~Y.}\binits{A.~Y.}},
\bauthor{\bsnm{Jordan},~\bfnm{M.~I.}\binits{M.~I.}} \AND
\bauthor{\bsnm{Weiss},~\bfnm{Y.}\binits{Y.}}
(\byear{2002}).
\btitle{On spectral clustering: Analysis and an algorithm}.
\bjournal{Adv. Neural Inf. Process. Syst.}
\bvolume{2}
\bpages{849--856}.
\end{barticle}
%
\bptok{imsref}%
\endbibitem

\bibitem{ozay}
%
\begin{bincollection}[auto:STB|2014/02/12|14:17:21]
\bauthor{\bsnm{Ozay},~\bfnm{N.}\binits{N.}},
\bauthor{\bsnm{Sznaier},~\bfnm{M.}\binits{M.}} \AND
\bauthor{\bsnm{Lagoa},~\bfnm{C.}\binits{C.}}
(\byear{2010}).
\btitle{Model (in) validation of switched arx systems with unknown
switches and its application to activity monitoring}.
In \bbooktitle{IEEE Conference on Decision and Control, CDC}
\bpages{7624--7630}.
\end{bincollection}
%
\bptok{imsref}%
\endbibitem

\bibitem{gpcaSDP}
%
\begin{bincollection}[auto:STB|2014/02/12|14:17:21]
\bauthor{\bsnm{Ozay},~\bfnm{N.}\binits{N.}},
\bauthor{\bsnm{Sznaier},~\bfnm{M.}\binits{M.}},
\bauthor{\bsnm{Lagoa},~\bfnm{C.}\binits{C.}} \AND
\bauthor{\bsnm{Camps},~\bfnm{O.}\binits{O.}}
(\byear{2010}).
\btitle{GPCA with denoising: A moments-based convex approach}.
In \bbooktitle{IEEE Conference on Computer Vision and Pattern
Recognition, CVPR}
\bpages{3209--3216}.
\bpublisher{IEEE Press}, \blocation{New York}.
\end{bincollection}
%
\bptok{imsref}%
\endbibitem

\bibitem{parsons2004subspace}
%
\begin{barticle}[auto:STB|2014/02/12|14:17:21]
\bauthor{\bsnm{Parsons},~\bfnm{L.}\binits{L.}},
\bauthor{\bsnm{Haque},~\bfnm{E.}\binits{E.}} \AND
\bauthor{\bsnm{Liu},~\bfnm{H.}\binits{H.}}
(\byear{2004}).
\btitle{Subspace clustering for high dimensional data: A review}.
\bjournal{ACM SIGKDD Explor. Newsl.}
\bvolume{6}
\bpages{90--105}.
\end{barticle}
%
\bptok{imsref}%
\endbibitem

\bibitem{ImprovedMUS}
%
\begin{binproceedings}[auto:STB|2014/02/12|14:17:21]
\bauthor{\bsnm{Rosenbaum},~\bfnm{M.}\binits{M.}} \AND
\bauthor{\bsnm{Tsybakov},~\bfnm{A.~B.}\binits{A.~B.}}
(\byear{2013}).
\btitle{Improved matrix uncertainty selector}.
In \bbooktitle{From Probability to Statistics and Back: High-Dimensional Models and Processes---A~Festschrift in Honor of Jon A. Wellner}
\bpages{276--290}.
\bpublisher{IMS}, \blocation{Beachwood,~OH}.
\end{binproceedings}
%
\bptok{imsref}%
\endbibitem


\bibitem{MUS}
%
\begin{barticle}[mr]
\bauthor{\bsnm{Rosenbaum},~\bfnm{Mathieu}\binits{M.}} \AND
\bauthor{\bsnm{Tsybakov},~\bfnm{Alexandre~B.}\binits{A.~B.}}
(\byear{2010}).
\btitle{Sparse recovery under matrix uncertainty}.
\bjournal{Ann. Statist.}
\bvolume{38}
\bpages{2620--2651}.
\bid{doi={10.1214/10-AOS793}, issn={0090-5364}, mr={2722451}}
\end{barticle}
%
\bptok{imsref}%
\endbibitem

\bibitem{ourSSC}
%
\begin{barticle}[mr]
\bauthor{\bsnm{Soltanolkotabi},~\bfnm{Mahdi}\binits{M.}} \AND
\bauthor{\bsnm{Cand{\'e}s},~\bfnm{Emmanuel~J.}\binits{E.~J.}}
(\byear{2012}).
\btitle{A geometric analysis of subspace clustering with outliers}.
\bjournal{Ann. Statist.}
\bvolume{40}
\bpages{2195--2238}.
\bid{doi={10.1214/12-AOS1034}, issn={0090-5364}, mr={3059081}}
\end{barticle}
%
\bptok{imsref}%
\endbibitem

\bibitem{RSCsupp}
%
\begin{bmisc}[auto:STB|2014/02/12|14:17:21]
\bauthor{\bsnm{Soltanolkotabi},~\bfnm{M.}\binits{M.}},
\bauthor{\bsnm{Elhamifar},~\bfnm{E.}\binits{E.}} \AND
\bauthor{\bsnm{Cand{\`e}s},~\bfnm{E.~J.}\binits{E.~J.}}
(\byear{2014}).
\bhowpublished{Supplement to ``Robust subspace clustering.''
DOI:\doiurl{10.1214/13-AOS1199SUPP}.}
\end{bmisc}
%
\bptok{imsref}%
\endbibitem

\bibitem{st2010l1}
%
\begin{barticle}[mr]
\bauthor{\bsnm{St{\"a}dler},~\bfnm{Nicolas}\binits{N.}},
\bauthor{\bsnm{B{\"u}hlmann},~\bfnm{Peter}\binits{P.}} \AND
\bauthor{\bparticle{van~de}~\bsnm{Geer},~\bfnm{Sara}\binits{S.}}
(\byear{2010}).
\btitle{{$\ell_1$}-penalization for mixture regression models}.
\bjournal{TEST}
\bvolume{19}
\bpages{209--256}.
\bid{doi={10.1007/s11749-010-0197-z}, issn={1133-0686}, mr={2677722}}
\bptnote{check related}%
\end{barticle}
%
\bptok{imsref}%
\endbibitem


\bibitem{sun2012scaled}
%
\begin{barticle}[mr]
\bauthor{\bsnm{Sun},~\bfnm{Tingni}\binits{T.}} \AND
\bauthor{\bsnm{Zhang},~\bfnm{Cun-Hui}\binits{C.-H.}}
(\byear{2012}).
\btitle{Scaled sparse linear regression}.
\bjournal{Biometrika}
\bvolume{99}
\bpages{879--898}.
\bid{doi={10.1093/biomet/ass043}, issn={0006-3444}, mr={2999166}}
\end{barticle}
%
\bptok{imsref}%
\endbibitem

\bibitem{MPPCA}
%
\begin{barticle}[mr]
\bauthor{\bsnm{Tipping},~\bfnm{Michael~E.}\binits{M.~E.}} \AND
\bauthor{\bsnm{Bishop},~\bfnm{Christopher~M.}\binits{C.~M.}}
(\byear{1999}).
\btitle{Probabilistic principal component analysis}.
\bjournal{J. R. Stat. Soc. Ser. B Stat. Methodol.}
\bvolume{61}
\bpages{611--622}.
\bid{doi={10.1111/1467-9868.00196}, issn={1369-7412}, mr={1707864}}
\end{barticle}
%
\bptok{imsref}%
\endbibitem

\bibitem{tomasi1992shape}
%
\begin{barticle}[auto:STB|2014/02/12|14:17:21]
\bauthor{\bsnm{Tomasi},~\bfnm{C.}\binits{C.}} \AND
\bauthor{\bsnm{Kanade},~\bfnm{T.}\binits{T.}}
(\byear{1992}).
\btitle{Shape and motion from image streams under orthography: A
factorization method}.
\bjournal{Int. J. Comput. Vis.}
\bvolume{9}
\bpages{137--154}.
\end{barticle}
%
\bptok{imsref}%
\endbibitem

\bibitem{ksubspaces}
%
\begin{barticle}[mr]
\bauthor{\bsnm{Tseng},~\bfnm{P.}\binits{P.}}
(\byear{2000}).
\btitle{Nearest {$q$}-flat to {$m$} points}.
\bjournal{J. Optim. Theory Appl.}
\bvolume{105}
\bpages{249--252}.
\bid{doi={10.1023/A:1004678431677}, issn={0022-3239}, mr={1757267}}
\end{barticle}
%
\bptok{imsref}%
\endbibitem

\bibitem{vidaltutorial}
%
\begin{barticle}[auto:STB|2014/02/12|14:17:21]
\bauthor{\bsnm{Vidal},~\bfnm{R.}\binits{R.}}
(\byear{2011}).
\btitle{Subspace clustering}.
\bjournal{IEEE Signal Process. Mag.}
\bvolume{28}
\bpages{52--68}.
\end{barticle}
%
\bptok{imsref}%
\endbibitem

\bibitem{GPCA}
%
\begin{barticle}[auto:STB|2014/02/12|14:17:21]
\bauthor{\bsnm{Vidal},~\bfnm{R.}\binits{R.}},
\bauthor{\bsnm{Ma},~\bfnm{Y.}\binits{Y.}} \AND
\bauthor{\bsnm{Sastry},~\bfnm{S.}\binits{S.}}
(\byear{2005}).
\btitle{Generalized principal component analysis (GPCA)}.
\bjournal{IEEE Trans. Pattern Anal. Mach. Intell.}
\bvolume{27}
\bpages{1945--1959}.
\end{barticle}
%
\bptok{imsref}%
\endbibitem

\bibitem{Yan}
%
\begin{bincollection}[auto:STB|2014/02/12|14:17:21]
\bauthor{\bsnm{Yan},~\bfnm{J.}\binits{J.}} \AND
\bauthor{\bsnm{Pollefeys},~\bfnm{M.}\binits{M.}}
(\byear{2006}).
\btitle{A general framework for motion segmentation: Independent,
articulated, rigid, nonrigid, degenerate and nondegenerate}.
In \bbooktitle{ECCV 2006}
\bpages{94--106}.
\end{bincollection}
%
\bptok{imsref}%
\endbibitem

\bibitem{montanariSC}
%
\begin{binproceedings}[auto:STB|2014/02/12|14:17:21]
\bauthor{\bsnm{Zhang},~\bfnm{A.}\binits{A.}},
\bauthor{\bsnm{Fawaz},~\bfnm{N.}\binits{N.}},
\bauthor{\bsnm{Ioannidis},~\bfnm{S.}\binits{S.}} \AND
\bauthor{\bsnm{Montanari},~\bfnm{A.}\binits{A.}}
(\byear{2012}).
\btitle{Guess who rated this movie: Identifying users through subspace clustering}.
In \bbooktitle{Proceedings of the International Conference on Uncertainty in Articial Intelligence}
\bpages{944--953}.
\end{binproceedings}
%
\bptok{imsref}%
\endbibitem

\bibitem{mediankflat}
%
\begin{bincollection}[auto:STB|2014/02/12|14:17:21]
\bauthor{\bsnm{Zhang},~\bfnm{T.}\binits{T.}},
\bauthor{\bsnm{Szlam},~\bfnm{A.}\binits{A.}} \AND
\bauthor{\bsnm{Lerman},~\bfnm{G.}\binits{G.}}
(\byear{2009}).
\btitle{Median $k$-flats for hybrid linear modeling with many outliers}.
In \bbooktitle{IEEE International Conference on Computer Vision Workshops, ICCV}
\bpages{234--241}.
\end{bincollection}
%
\bptok{imsref}%
\endbibitem

\bibitem{zhang2012hybrid}
%
\begin{barticle}[mr]
\bauthor{\bsnm{Zhang},~\bfnm{Teng}\binits{T.}},
\bauthor{\bsnm{Szlam},~\bfnm{Arthur}\binits{A.}},
\bauthor{\bsnm{Wang},~\bfnm{Yi}\binits{Y.}} \AND
\bauthor{\bsnm{Lerman},~\bfnm{Gilad}\binits{G.}}
(\byear{2012}).
\btitle{Hybrid linear modeling via local best-fit flats}.
\bjournal{Int. J. Comput. Vis.}
\bvolume{100}
\bpages{217--240}.
\bid{doi={10.1007/s11263-012-0535-6}, issn={0920-5691}, mr={2979307}}
\end{barticle}
%
\bptok{imsref}%
\endbibitem

\bibitem{MCdata}
%
\begin{bincollection}[auto:STB|2014/02/12|14:17:21]
\bauthor{\bsnm{Zhou},~\bfnm{F.}\binits{F.}},
\bauthor{\bsnm{Torre},~\bfnm{F.}\binits{F.}} \AND
\bauthor{\bsnm{Hodgins},~\bfnm{J.~K.}\binits{J.~K.}}
(\byear{2008}).
\btitle{Aligned cluster analysis for temporal segmentation of human motion}.
In \bbooktitle{IEEE International Conference on Automatic Face and
Gesture Recognition, FG}
\bpages{1--7}.
\end{bincollection}
%
\bptok{imsref}%
\endbibitem

\bibitem{MSweb}
%
\begin{bmisc}[auto:STB|2014/02/12|14:17:21]
\bhowpublished{\href{http://www.stanford.edu/\textasciitilde mahdisol/RSC}{www.stanford.edu/\textasciitilde mahdisol/RSC}}.
\end{bmisc}
%
\bptok{imsref}%
\endbibitem

\bibitem{homotopy}
%
\begin{bmisc}[auto:STB|2014/02/12|14:17:21]
\bhowpublished{\href{http://users.ece.gatech.edu/\textasciitilde sasif/homotopy}{users.ece.gatech.edu/\textasciitilde sasif/homotopy}}.
\end{bmisc}
%
\bptok{imsref}%
\endbibitem

\end{thebibliography}
\end{document}